\documentclass[pdflatex,sn-nature]{sn-jnl}

\usepackage{amsmath}
\usepackage{bm}
\usepackage[title]{appendix}%
\usepackage{xcolor}%
\usepackage{manyfoot}%
\usepackage{listings}%
\usepackage{algorithm}
\usepackage{algorithmicx}
\usepackage{algpseudocode}
\usepackage{comment}
\usepackage{picinpar}
\usepackage{flushend}
\usepackage{colortbl}
\usepackage{soul}
\usepackage{multirow}
\usepackage{pifont}
\usepackage{alltt}
\usepackage{url}
\usepackage{pbox}
\usepackage{footmisc} 
\usepackage{threeparttable}
\usepackage{siunitx}
\usepackage{tabularx}
\usepackage{multirow}
\usepackage{booktabs}
\usepackage{colortbl}
\definecolor{tabcolor}{rgb}{.6,.32,.8}
\usepackage{newtxtext,newtxmath}
\usepackage{enumerate}
\usepackage{natbib}
\usepackage{bibentry}
\usepackage{caption}
\usepackage{siunitx}
\usepackage{pifont}

\newtheorem{remark}{Remark}
\newtheorem{assumption}{Assumption}
\newtheorem{lemma}{Lemma}
\newtheorem{proposition}{Proposition}

\newtheorem{theorem}{Theorem}

\makeatletter
\renewcommand{\fnum@figure}{\textbf{Fig. \thefigure}}
\renewcommand{\fnum@table}{\textbf{Table \thetable}}
\makeatother

\newcommand{\upcite}[1]{\textsuperscript{\citep{#1}}}
\usepackage[numbers]{natbib}
\bibpunct{}{}{,}{n}{}{}  

\unnumbered

\begin{document}
	
	\title[Article Title]{
		Reinforcement learning in linear embedding space unlocks generalizable control across soft robot configurations
	}
	
	
	\author*[1]{\fnm{Xinglong} \sur{Zhang}}
	\equalcont{These authors contributed equally to this work.}
	\author[1]{\fnm{Cong} \sur{Li}}
	\equalcont{These authors contributed equally to this work.}
	\author[2]{\fnm{Hangjie} \sur{Mo}}
	\equalcont{These authors contributed equally to this work.}
	\author[1]{\fnm{Yue} \sur{Jiang}}
	\author[1]{\fnm{Wenyu} \sur{Cao}}
	\author*[1]{\fnm{Xin} \sur{Xu}}
	\author[1]{\fnm{Wei} \sur{Jiang}}
	\author[3,4]{\fnm{Zhenshan} \sur{Bing}}
	\author[1]{\fnm{Yihe} \sur{Yang}}
	\author[2]{\fnm{Xiaojian} \sur{Li}}
	\author[5]{\fnm{Yueneng} \sur{Yang}}
	\author[1]{\fnm{Huimin} \sur{Lu}}
	\author[1]{\fnm{Ling-li} \sur{Zeng}}
	\author[4]{\fnm{Alois} \sur{Knoll}}
	\author*[1]{\fnm{Dewen} \sur{Hu}}
	\author[6]{\fnm{Li} \sur{Wen}}
	\author[7]{\fnm{Wei} \sur{Pan}}
	
	\affil*[1]{\orgdiv{College of Intelligence Science and Technology}, \orgname{National University of Defense Technology}, \orgaddress{\street{Deya Road}, 
			\city{Changsha}, \postcode{410003}, \state{Hunan}, \country{China}}}
	
	\affil[2]{\orgdiv{ School of Management}, \orgname{Hefei University of Technology}, \orgaddress{\street{Nanyihuan Road}, \city{Hefei}, \postcode{230002}, \state{Anhui}, \country{China}}}
	\affil[3]{\orgdiv{State Key Laboratory for Novel Software Technology and the School of Science and Technology}, \orgname{Nanjing University (Suzhou Campus)}, \orgaddress{\street{Taihu Road}, \city{Suzhou}, \postcode{215163}, \country{China}}}
	\affil[4]{\orgdiv{School of Computation, Information and Technology}, \orgname{Technical University of Munich}, \orgaddress{\street{Boltzmann Strasse}, \city{Munich}, \postcode{85748}, \country{Germany}}}
	
	\affil[5]{\orgdiv{College of Aerospace Science and Engineering}, \orgname{National University of Defense Technology}, \orgaddress{\street{Deya Road}, 
			\city{Changsha}, \postcode{410003}, \state{Hunan}, \country{China}}}
	
	\affil[6]{\orgdiv{School of Mechanical Engineering and Automation}, \orgname{Beihang University}, \orgaddress{\street{Changping}, 
			\city{Beijing}, \postcode{100191},  \country{China}}}
	\affil[7]{\orgdiv{School of Engineering}, \orgname{Newcastle University}, \orgaddress{\street{Claremont Road}, \city{Newcastle-upon-Tyne},\\  \postcode{NE1 7RU}, \country{UK}}} 

\abstract{
Soft-bodied organisms such as octopuses and elephant trunks exhibit remarkable morphological adaptability, dynamically reconfiguring body shape and stiffness, and flexibly adjusting their control strategies to enable versatile behaviors\upcite{li2021self, wang2024spirobs}. Inspired by these biological systems, various soft robots have emerged in recent decades, featuring diverse materials, stiffnesses, and morphologies tailored to specific tasks.
Despite substantial advances in the materials and structural designs of soft robots\upcite{xie2023octopus, rus2015design,xu2025transforming}, developing a generalizable control framework capable of rapid adaptation across diverse configurations remains a long-standing challenge. Existing controllers are limited to fixed configurations, demanding laborious configuration-specific remodelling and policy redesign for new configurations\upcite{pique2022controlling,haggerty2023control,mamakoukas2021derivative}.
Here, we introduce a generalizable control system that enables rapid adaptation across diverse soft robot configurations via reinforcement learning in a shared linear Koopman embedding space. By encoding robot dynamics into this embedding space, our method decouples control policies from specific morphologies, allowing real-time, model-free policy adaptation across diverse configurations without retraining from scratch.
We validate our system across 30 soft robot configurations with different materials, stiffnesses, structural geometries, and assembly patterns, achieving substantial performance gains, including a 75× reduction in training samples compared to state-of-the-art methods, while sustaining robust performance under high-speed (\SI{1.89}{\, \meter \cdot \second^{-1}}), high-payload (\SI{1}{\, \kilogram}) and multiactuator faults. Our system enables a single base policy, with online updates, to achieve a wide range of real-world skills previously unattainable in soft robotics, such as forceful nail hammering, delicate drinks serving, fluid calligraphy brush handwriting, and quick hand-eye reaction games. This work establishes a unified and adaptable control paradigm for diverse soft robot configurations, bridging mechanical reconfigurability with control flexibility, and may offer broader insights for generalizable control in complex physical systems.
}

\maketitle
 Natural soft-bodied organisms demonstrate remarkable morphological adaptability through dynamic reconfiguration
 and adaptive control\upcite{li2021self,wang2024spirobs,pan2025miniature,yue2025embodying}. For instance, octopuses reshape their body structures and elephant trunks modulate stiffness to enable versatile functionality that ranges from agile locomotion to flexible manipulation. This adaptability stems from embedded sensorimotor loops that continuously adjust to structural changes, allowing seamless multitask execution.
 Inspired by biological systems, soft robots leverage compliant materials and structures to achieve inherent adaptability and safe interaction\upcite{manti2015bioinspired,ren2024design,xie2023octopus,kim2013soft,shah2021soft}. 
 These features enable applications in delicate manipulation\upcite{manti2015bioinspired}, human-robot cooperation\upcite{rus2015design}, and confined-space navigation\upcite{mo2024data}. 
 Despite advances in material design and structural reconfiguration of soft robots to meet diverse task demands\upcite{freeman2023topology,li2022scaling,xu2025transforming}, developing a generalizable control framework capable of fast adaptation across diverse configurations remains a long-standing problem\upcite{freeman2023topology,milana2025physical}. Existing controllers are typically configuration-specific, requiring extensive remodelling or controller redesign for each configuration\upcite{pique2022controlling,haggerty2023control,mamakoukas2021derivative,Liu11011930}, limiting the scalable deployment of soft robots. 

One challenge for multi-configuration adaptation is the domain expertise and intensive labor required for dynamical modelling of soft robots\upcite{armanini2023soft}. 
Analytical modelling methods, such as the Cosserat rod theory, achieve high modelling accuracy under quasi-static conditions for specific configurations but come with prohibitive computational costs\upcite{renda2018discrete,till2019real}.  
Data-driven modelling approaches, such as neural network and Gaussian process regression, reduce the need for prior dynamic knowledge, but require extensive data collection and training\upcite{huang2024high,braganza2007neural,giorelli2015neural,kim2021probabilistic}.
The Koopman operator\upcite{korda2018linear,mamakoukas2023learning} offers a promising alternative by capturing nonlinear dynamics through linear evolutions in an infinite-dimensional embedding space. 
Finite-dimensional approximations of Koopman models are commonly integrated with the linear quadratic regulator (LQR)\upcite{haggerty2023control} and model predictive control (MPC)\upcite{zhang2021}, enabling efficient controller synthesis.
Among them, Koopman-based MPC has demonstrated its effectiveness in the control of soft robots, but mainly in quasi-static space\upcite{Bruder9477047,bruder2020data,pan2023auto,wang2022improved}.  A recent configuration-specific Koopman-based control approach advanced a soft robotic arm operating in an inertial dynamic regime through time-delay embeddings\upcite{haggerty2023control}. Despite these advances, the aforementioned approaches require repetitive modelling and controller design to adapt to new configurations, limiting their practical applicability in reconfigurable robotic systems.

The second challenge is the rapid generation of policies to accommodate dynamic variations in multiple configurations. 
Although reinforcement learning (RL) offers a promising model-free policy design approach\upcite{kober2013reinforcement,meng2025preserving,medany2025model,han2024lifelike}, prevalent RL algorithms face high sample complexity, particularly in continuous control tasks\upcite{naughton2021elastica,bing2024context}.
This high sample demand is costly for real-world training, due to the delicate nature and susceptibility to wear and tear of soft robots\upcite{thuruthel2018model}.
Although physical simulators provide an alternative to real-world training, RL algorithms such as deep deterministic policy gradient\upcite{satheeshbabu2020continuous} and proximal policy optimization\upcite{nazeer2024rl} still require millions of data to converge\upcite{naughton2021elastica}.
Moreover, building high-fidelity simulators of soft robots is challenging due to their nonlinear, configuration-dependent dynamics, exacerbating the sim-to-real gap. 
 \begin{figure*}[!t]
    \centering
    \includegraphics[width=1 \textwidth]{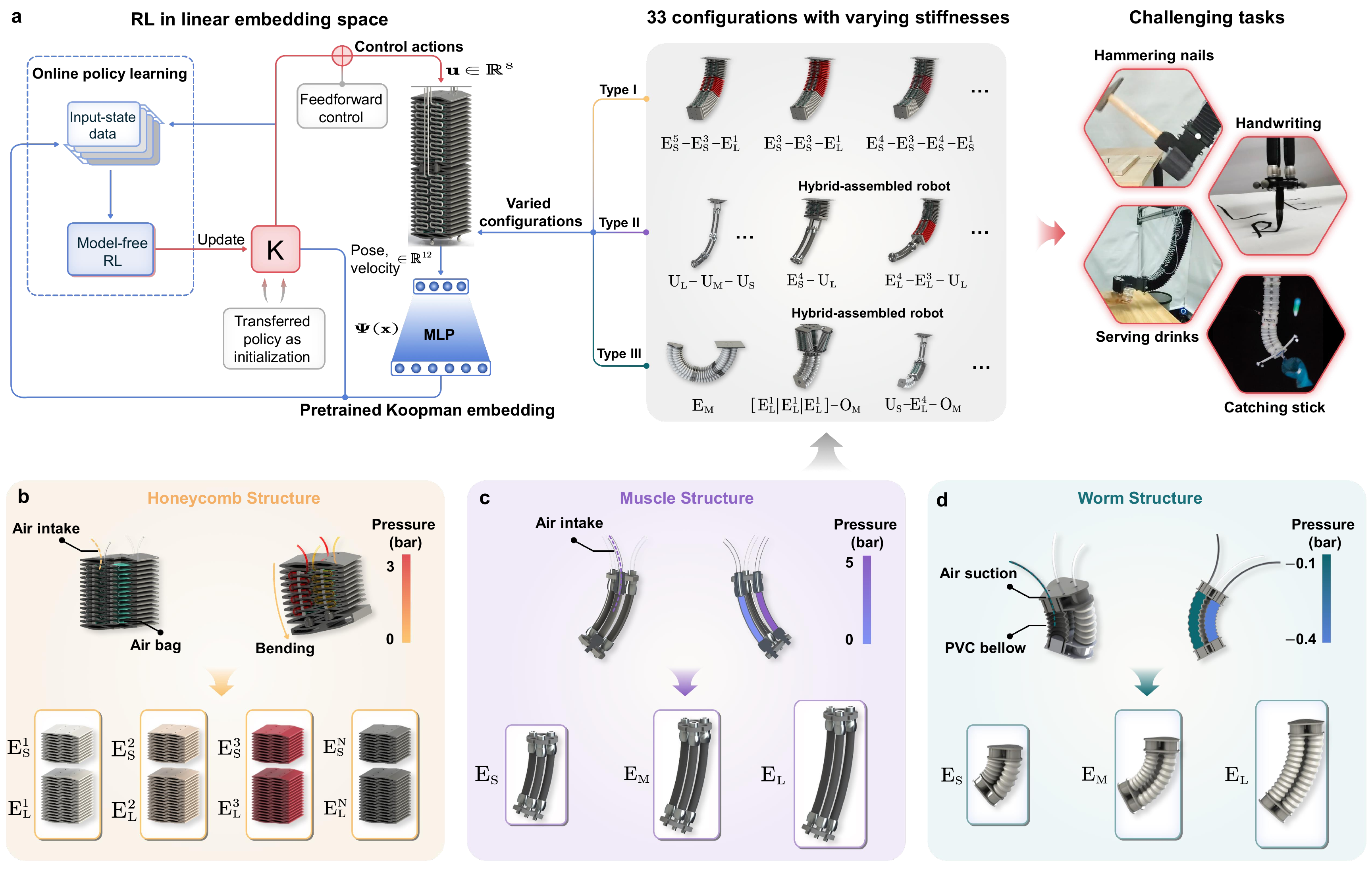}
    \caption{
   \textbf{Our generalizable control system across 30 configurations with diverse materials and stiffnesses. }
   \textbf{a}, 
   A generalizable control system with fast adaptability across multiple configurations and tasks. 
   The Koopman embedding $\Psi(x)$ is offline trained only once on a fixed configuration, and the policy gain matrix $K$ is online updated for multi-configuration adaptation via model-free RL in the Koopman embedding space.
   \textbf{b}, Type I elephant-trunk robots:
  Each segment features a honeycomb structure with four pneumatic actuators. 
   Honeycomb segments' properties vary with wall thicknesses (superscript indices $1\cdots N$ denote wall thickness affecting the stiffness) and lengths (subscript indices ‘S, M, L’ represent segment lengths: Short, Medium, and Long). 
      \textbf{c}, Type II  soft-muscle robots:
   Each segment is composed of three parallel pneumatic muscles, each driven by a pneumatic actuator.
   \textbf{d}, Type III  worm robots:
   Each segment features a worm structure with three pneumatic actuators. 
   }
    \label{fig natural}
\end{figure*}

In this work, we introduce a generalizable control system that enables fast adaptation across diverse configurations in soft robots (Fig.~\ref{fig natural}). Our system consists of two key modules:  (a) a linear embedding RL (LERL) framework that learns control policies within a linear Koopman embedding space for fast multi-configuration adaptation (Fig.~\ref{fig natural}a), and (b) three types of soft robot platforms with diverse configurations for experimental validation, including an elephant-trunk-shaped robot, a soft-muscle robot, and a worm-like robot (Fig.~\ref{fig natural}b-d). These robot configurations feature diverse materials, stiffnesses, structural geometries, and assembly patterns.
The presented LERL framework is built upon the hypothesis that different soft robot configurations share fundamental but also exhibit configuration-dependent dynamic characteristics. We demonstrate that by mapping these shared dynamics into a unified linear embedding space and performing policy adaptation online within this space, control policies could be quickly transferred across diverse configurations.

We validated our system experimentally across 30 robot configurations, including 22 variations of elephant-trunk robots, three soft-muscle robot variants, three worm-like robots, and two hybrid-assembled robots combining honeycomb and soft-muscle segments (Fig.~\ref{fig natural}a and Fig.~\ref{fig all config}).
The experimental results demonstrate that LERL achieves a 75× reduction in training samples while maintaining robust performance in challenging conditions, including high-payload manipulation (\SI{1}{\, \kilogram}), high-speed motion (Worm: $\SI{1.89}{\, \meter \cdot \second^{-1}}$; Elephant-trunk: $\SI{0.78}{\, \meter \cdot \second^{-1}}$), and multi-actuator fault scenarios (Table~\ref{tab related works}). 
Critically, LERL enables rapid policy transfer to execute diverse tasks (Fig.~\ref{fig natural}d), allowing a base policy with online updates to empower robots performing previously unattainable real-world skills, including carpenter-inspired nail hammering into wooden boards, bartender-like drinks serving with safe human-robot interaction, fluid calligraphy brush handwriting, and quick hand-eye reaction games mimicking athletic and child training.
The presented work represents a step toward a generalizable control framework that enables rapid multi-configuration adaptation in soft robots. 
Through online data-efficient RL within a pretrained Koopman embedding space, our proposed system decouples control policies from specific configurations, bridging mechanical reconfigurability and control flexibility. 
This work may also inspire policy transfer and adaptation solutions for other physical systems with shared fundamental dynamic characteristics.

			\begin{figure}[t!]
				\centering				\includegraphics[width=0.95\textwidth]{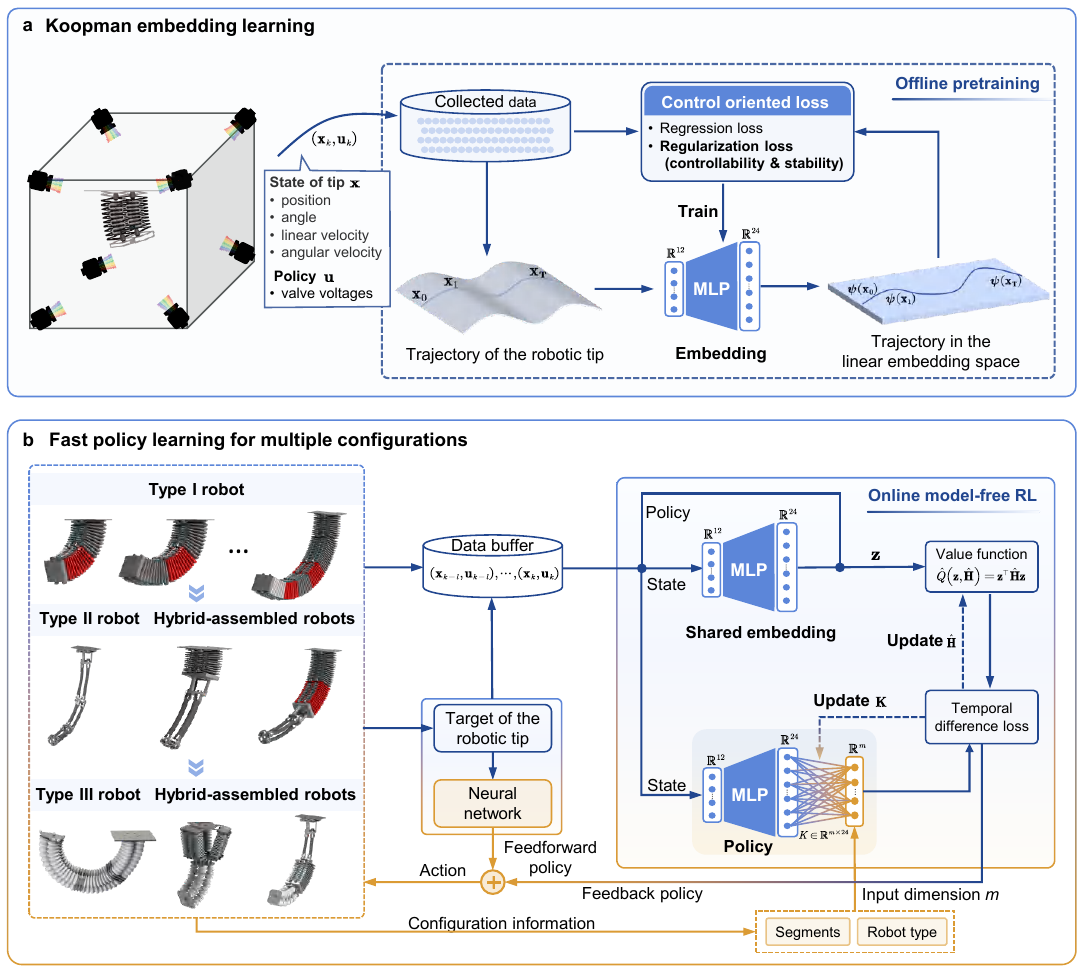}
				\caption{ \textbf{Linear embedding RL framework.}
    \textbf{a}, The Koopman embedding is trained offline with precollected input-state samples from the single honeycomb segment $\text{E}^{1}_\text{S}$.
    \textbf{b}, Online policy learning is conducted in the linear embedding space using model-free RL across multiple configurations with diverse materials, stiffnesses, structural geometries, and assembly pattern, where the pretrained Koopman embedding is shared for all new configurations.
    }
				\label{fig:control-structure}
			\end{figure}

\begin{table*}[h!tb]
	\captionsetup{justification=justified,singlelinecheck=false}
	\caption{Comparison with state-of-the-art (SOTA) methods for control of soft robots.
		Our LERL framework enables policy transfer across 30 distinct configurations, whereas most SOTA methods are limited to a fixed configuration. With just 2k samples for policy transfer, LERL uses substantially fewer resources than RL methods such as proximal policy optimization (PPO)\upcite{nazeer2024rl} ($\geq$5250k, 2600× more), trust region policy optimization (TRPO)\upcite{centurelli2022closed} ($\geq$300k, 150× more), and deep deterministic policy gradient (DDPG)\upcite{satheeshbabu2020continuous} ($\geq$100k, 50× more). Additionally, KCID with time-delay (TD) embeddings\upcite{haggerty2023control} requires 18k samples per configuration, which is 9× more than LERL with MLP embeddings. When both using the same MLP embeddings, KCID demands 75× more samples than LERL. Despite using the fewest samples, LERL outperforms all other SOTA methods in terms of speed, acceleration, and payload capacity. LSTM: Long short-term memory;  
		RNN: Recurrent neural network. }
	\label{tab related works}
	\vskip 0.2cm
	\begin{threeparttable}
			\renewcommand{\arraystretch}{1.2}
			\scalebox{0.7}{
			\begin{tabular}{lcccccccccccc}
				\toprule
				\multirow{2}{*}{\textbf{Work}}
				&\multirow{2}{*}{\textbf{Method}} 
				&\multirow{2}{*}{\textbf{Config.}} 
				&\multirow{1}{*}{\textbf{Config.}} 
				&\multirow{1}{*}{\textbf{Robot}} 
				&\multirow{1}{*}{\textbf{Hybrid}} 
				&\multirow{1}{*}{\textbf{Online}} 
				&\multirow{2}{*}{\textbf{Sample}}
				&\multirow{2}{*}{\textbf{Speed} $\left(\text{m}\cdot \text{s}^{-1}\right)$}
				&\multirow{2}{*}{\textbf{Accel} $\left(\text{m}\cdot \text{s}^{-2}\right)$}
				&\multirow{2}{*}{\textbf{Max Load} (g)}
				\\  
				&
				&
				&\multirow{1}{*}{\textbf{transfer}} 
				&\multirow{1}{*}{\textbf{type}}
				&\multirow{1}{*}{\textbf{assembly}}
				& \multirow{1}{*}{\textbf{learning}}
				& 
				\\
				\midrule
				\multirow{2}{*}{Ours} 
				&\multirow{2}{*}{LERL}
				& \multirow{2}{*}{30}
				& \multirow{2}{*}{29}
				&\multirow{2}{*}{3}
				&\multirow{2}{*}{\ding{51}}
				&\multirow{2}{*}{\ding{51}}
				&\multirow{2}{*}{2k} 
				&\multirow{1}{*}{Type I: 0.78} 
				&\multirow{1}{*}{11.9} 
				&\multirow{1}{*}{1000}  \\  \cmidrule{9-11}  
				& 
				& 
				&
				&
				& 
				&
				&\multirow{1}{*}{(transfer)} 
				&\multirow{1}{*}{Type III: 1.89} 
				&\multirow{1}{*}{22.34} 
				&\multirow{1}{*}{200}
				\\ 
				\midrule
				\multirow{1}{*}{Ref\upcite{centurelli2022closed}} 
				&\multirow{1}{*}{LSTM + TRPO}
				& 1
				& -
				&\multirow{1}{*}{1}
				&\multirow{1}{*}{\ding{55}}
				&\multirow{1}{*}{\ding{55}}
				& $\geq$300k
				&0.17
				&-
				& 165
				\\ 
				\midrule
				\multirow{1}{*}{Ref\upcite{thuruthel2018model}} 
				&\multirow{1}{*}{RNN + RL} 
				& 1
				& -
				&\multirow{1}{*}{1}
				&\multirow{1}{*}{\ding{55}}
				&\multirow{1}{*}{\ding{55}}
				& 28k
				& -
				& - 
				& 105
				\\ 
				\midrule
				\multirow{1}{*}{Ref\upcite{nazeer2024rl}} 
				&\multirow{1}{*}{LSTM+PPO} 
				& 1
				& -
				&\multirow{1}{*}{1}
				&\multirow{1}{*}{\ding{55}}
				&\multirow{1}{*}{\ding{51}}
				& $\geq$ 5250k 
				&0.3
				&-
				& -
				\\ 
				\midrule
				\multirow{1}{*}{Ref\upcite{satheeshbabu2020continuous}} 
				&\multirow{1}{*}{DDPG} 
				& 1
				& -
				&\multirow{1}{*}{1}
				&\multirow{1}{*}{\ding{55}}
				&\multirow{1}{*}{\ding{55}}
				& $\geq$ 100k 
				& -
				& -
				& 20
				\\ 
				\midrule
				\multirow{1}{*}{Ref\upcite{jiang2021hierarchical}} 
				&\multirow{1}{*}{$\mathcal{Q}$-learning} 
				& 1
				& -
				&\multirow{1}{*}{1}
				&\multirow{1}{*}{\ding{55}}
				&\multirow{1}{*}{\ding{55}}
				& 20k 
				& -
				& -
				&500
				\\ 
				\midrule
				\multirow{2}{*}{Ref\upcite{haggerty2023control}} 
				&\multirow{2}{*}{KCID} 
				& \multirow{2}{*}{2}
				& \multirow{2}{*}{-}
				&\multirow{2}{*}{1}
				&\multirow{2}{*}{\ding{55}}
				&\multirow{2}{*}{\ding{55}}
				&\multirow{1}{*}{TD: 18k} 
				&\multirow{2}{*}{1.52}
				&\multirow{2}{*}{11.6}
				&\multirow{2}{*}{-}
				\\ \cmidrule{8-8}
				&&&&&&&\multirow{1}{*}{MLP: 150k} &&&\\
				\midrule
				\multirow{1}{*}{Ref\upcite{wang2022improved}} 
				&\multirow{1}{*}{Koopman + MPC} 
				& 1
				& -
				&\multirow{1}{*}{1}
				&\multirow{1}{*}{\ding{55}}
				&\multirow{1}{*}{\ding{55}}
				& 247k 
				&-
				&-
				& -
				\\ 
				\midrule
				\multirow{1}{*}{Ref\upcite{bruder2020data}} 
				&\multirow{1}{*}{Koopman  + MPC} 
				& 1
				& -
				&\multirow{1}{*}{1}
				&\multirow{1}{*}{\ding{55}}
				&\multirow{1}{*}{\ding{55}}
				& $\geq$43k
				& -
				&-
				& -
				\\ 
				\midrule
				\multirow{1}{*}{Ref\upcite{Bruder9477047}} 
				&\multirow{1}{*}{Koopman + MPC} 
				& 1
				& -
				&\multirow{1}{*}{1}
				&\multirow{1}{*}{\ding{55}}
				&\multirow{1}{*}{\ding{55}}
				& $\geq$325k
				& -
				& -
				& 275
				\\ 
				\midrule
				\multirow{1}{*}{Ref\upcite{huang2024high}} 
				&\multirow{1}{*}{Hybrid model + PID} 
				&1
				& -
				&\multirow{1}{*}{1}
				&\multirow{1}{*}{\ding{55}}
				&\multirow{1}{*}{\ding{55}}
				& 120k
				&0.08
				&0.07
				& -
				\\ 
				\bottomrule
			\end{tabular}
		}	 
	\end{threeparttable}
\end{table*}

\section*{Linear embedding RL framework}
LERL formulates control policy learning in a pretrained linear embedding space, allowing computationally and data efficient online policy learning across diverse configurations (Fig.~\ref{fig:control-structure}).
The initial step involves pretraining a multilayer perceptron (MLP)-based Koopman embedding $\Psi(x)$ using input-state data collected from a single honeycomb segment ($\text{E}^{1}_\text{S}$), where $x\in\mathbb{R}^{12}$ represents the pose (position and orientation) and velocity (linear and angular velocities) of the robot tip (Methods and Supplementary Section 1). The Koopman embedding maps the evolution of the robot tip's pose and velocity to a higher-dimensional, inherently linear space (Fig.~\ref{fig:control-structure}a).
%
In Koopman embedding learning, we introduce control-oriented regularization techniques along with regression objectives to ensure key open-loop control properties (stability and controllability) in the embedding space,  which are crucial for closed-loop performance. 

With pretrained Koopman embedding, model-free RL is employed to update policies online in the linear embedding space (Fig.~\ref{fig:control-structure}b). The Koopman embedding and control policy are then transferred across configurations, facilitating real-time policy adaptation for diverse configurations in the unified linear embedding space.
To address different adaptation scenarios, we introduce two variants of our method: LERL and LERL-transfer. LERL trains the policy from scratch using a pretrained configuration-specific Koopman embedding. 
LERL-transfer extends LERL by incorporating policy transfer, beginning with a pretrained policy and Koopman embedding from one configuration, and subsequently updating the policy online to adapt to new configurations.
\section*{Results}
In this section, we conducted comprehensive experiments to validate the generalizability and adaptability of our control system. The first experiment was to validate the pretrained Koopman embedding and online policy learning performance on one configuration. In subsequent experiments, the policy was transferred to control 30 distinct configurations, with comprehensive validations across all cases.  Throughout the experimental study, we compared our approach with the Koopman control method with inertial dynamics (KCID)\upcite{haggerty2023control}, a state-of-the-art solution for dynamic control of soft robots with fixed configurations. KCID is limited by its reliance on configuration-specific modelling of soft robots. In contrast, our method is configuration-agnostic, achieving a 75× reduction in training samples and eliminating the bottlenecks of repetitive modelling when adapting to new configurations. Finally, the control policy was transferred across diverse configurations to accomplish real-word challenging tasks previously unattainable for soft robots, including carpenter-inspired nail hammering, bartender-like drinks serving, fluid calligraphy brush handwriting, and quick hand-eye reaction games.

%
%
\begin{figure*}[t!]
    \centering
    \includegraphics[width=0.98 \textwidth]{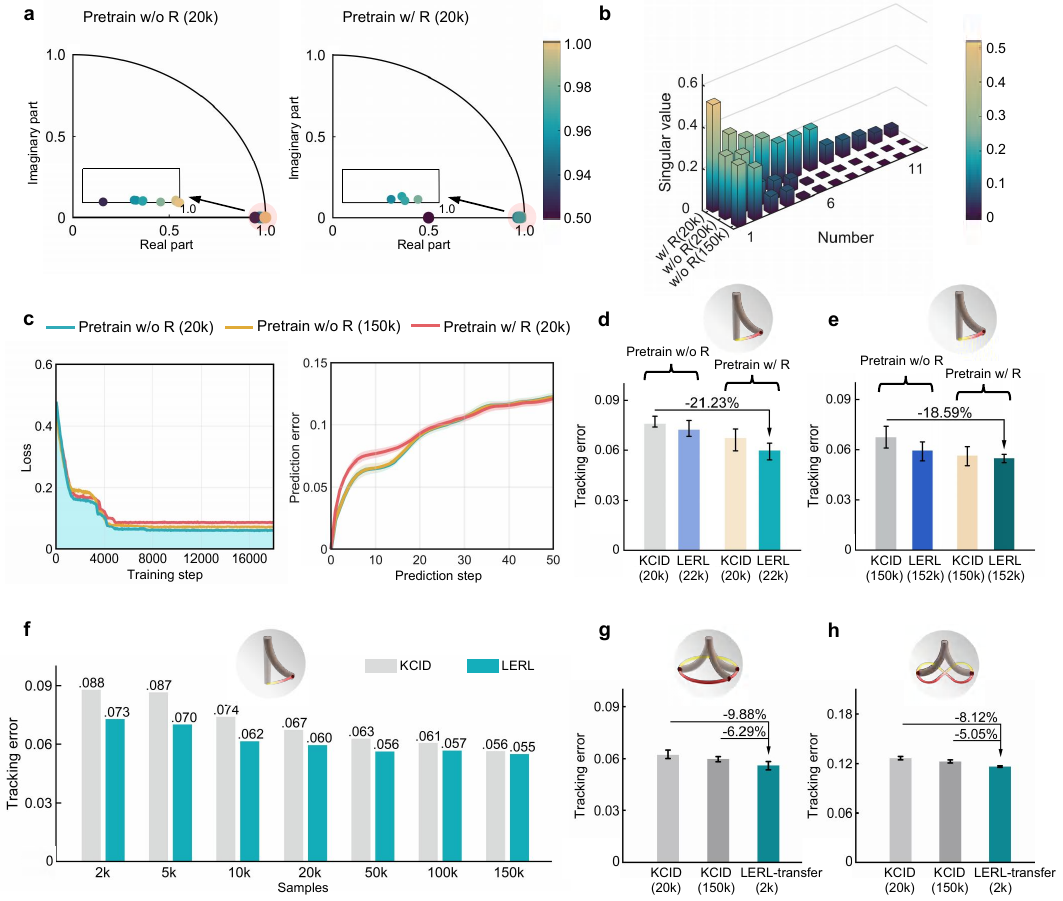}
    \caption{
    \textbf{Offline Koopman embedding training 
    and online policy learning for honeycomb segment $\text{E}^{1}_\text{S}$.}
 \textbf{a}, 
 The open-loop eigenvalue plots for Koopman models pretrained with (w/) or without (w/o) regularization (R).  The eigenvalues pretrained w/ regularization locate farther from the unit circle, resulting in a larger stability margin; 20k=20000.
  \textbf{b}, The singular value distribution  (SVD) of the controllability matrix of Koopman models. Pretraining w/ regularization leads to SVDs with higher magnitudes. 
  \textbf{c}, The training loss (equation~\eqref{eq loss sum} in Methods) and  prediction error of the Koopman model, $\text{Prediction\,error}:= \sqrt{\frac{1}{L} \sum_{l=1}^L \left\| x_{l} - \hat{x}_{l} \right\|^2}$, $L$ denotes the amount of samples. 
    \textbf{d}-\textbf{e}, Performance comparison of KCID and LERL during pretraining w/ or w/o regularization, 
$\text{Tracking error} := \frac{1}{M \cdot N} \sum_{i=1}^{M} \sum_{k=1}^{N}  \left\| \bar{x}_{i,k} - \bar{x}_{r, i,k} \right\|^2$, 
   $\bar{x}$, $\bar{x}_{r} \in \mathbb{R}^{12}$ are the normalized real state variable (pose and velocity) and its reference, $N$ is the algorithm horizon, and $M = 5$ is the number of experiments.
\textbf{f}, Sample comparison of LERL and KCID in target reaching.
 \textbf{g}-\textbf{h}, Performance comparison in policy transfer to circular and bow tie tracking.
    }
    \label{offline-online}
\end{figure*}
  \begin{figure*}[t!]
    \centering
    \includegraphics[width=0.95 \textwidth]
    {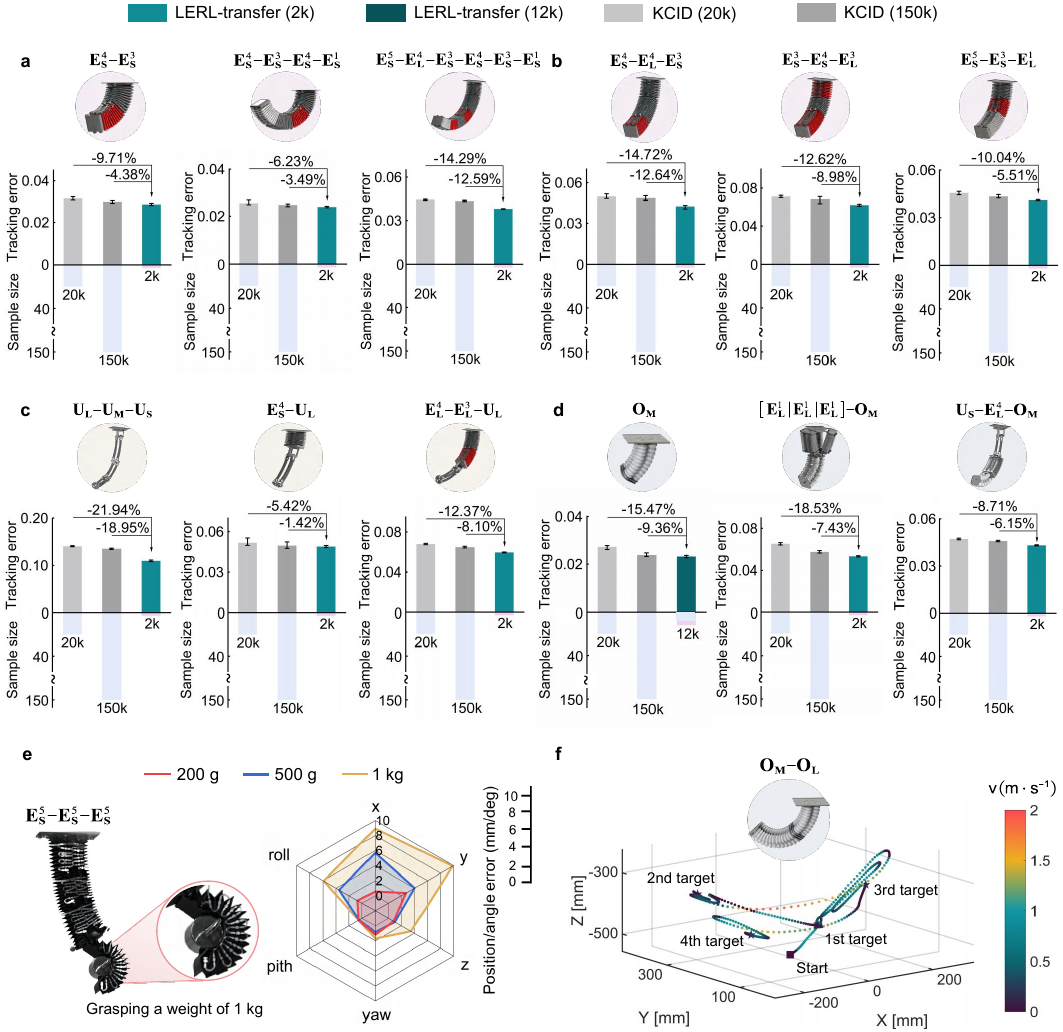}
    \caption{
    \textbf{ 
   Policy adaptation performance to multiple configurations of three soft robot prototypes.
}
\textbf{a}, Policy adaptation to elephant-trunk robots with multiple segments. 
\textbf{b}, Policy adaptation to elephant-trunk robots under diverse assembly configurations.
\textbf{c}, Policy adaptation to soft-muscle and hybrid-assembled robots.
\textbf{d}, Policy adaptation to worm robots.
\textbf{e}, Empirical robustness under different loads (\SI{200}{\, \gram}, \SI{500}{\, \gram}, \SI{1}{\, \kilogram}) and visualization of the robot grasping a \SI{1}{\, \kilogram} weight to the target position.
\textbf{f}, The worm robot in configuration $\text{E}_\text{M}$-$\text{E}_\text{L}$ tracks targets with a maximum velocity up to $\SI{1.89}{\, \meter \cdot \second^{-1}}$.
}
    \label{fig transfer diveser confg}
\end{figure*}
\subsection*{Policy learning performance in a single configuration}
\vspace{-2mm}

An MLP-based Koopman embedding $\Psi(x)$ was first trained in a configuration (honeycomb segment $\text{E}^{1}_\text{S}$) with control-oriented regularization  (Methods and Supplementary Section 1). 
    As demonstrated in Fig.~\ref{offline-online}a, applying control-oriented regularization during the pretraining phase produces an open-loop stable Koopman model, characterized by smaller-magnitude eigenvalues positioned farther from the unit circle compared to unregularized models. In addition, it preserves state prediction errors comparable to the unregularized case (Fig.~\ref{offline-online}c), while enlarging the singular values of the controllability matrix  (Fig.~\ref{offline-online}b). The enhanced open-loop stability and controllability address the gap between data-driven modelling and closed-loop control.

%
%

With the pretrained Koopman embedding, model-free RL is employed to update policies online in the linear embedding space. 
The control policies within the pretrained embedding space were generated using 2000 (2k) samples: 500 samples for offline feedforward policy generation (Fig.~\ref{fig:control-structure}b; Methods and Supplementary Section 2) and 1.5k samples for online feedback policy training (Methods and Supplementary Section 4). 
As demonstrated in Fig.~\ref{offline-online}, LERL using 22k training samples, 20k for offline Koopman embedding training and 2k for policy generation (rightmost bar, Fig.~\ref{offline-online}d), achieves an 11.37\% reduction in tracking error compared to KCID with 150k samples (leftmost bar, Fig.~\ref{offline-online}e); and results in a steady-state position error of 0.32 mm, nearly 4× smaller than KCID's 1.26 mm.
    The comparison under multiple sample conditions in Fig.~\ref{offline-online}f demonstrates a substantial sample reduction of LERL over KCID even when the configuration is unchanged. Specifically, LERL with 20k pretraining samples performs on par with KCID using 100k pretraining samples, reducing training data by 5× in a fixed configuration. For new tasks, LERL-transfer can further refine the transferred policy online with just 2k samples, surpassing KCID in both control performance and sample efficiency (Figs.~\ref{offline-online}g-h).

\subsection*{Policy adaptation performance in multiple configurations}
\vspace{-2mm}
The Koopman embedding $\Psi(x)$, pretrained in the configuration $\text{E}^{1}_\text{S}$ using 20k training samples,  was reused in new configurations without retraining. The learned policy was transferred and updated online to adapt to dynamic variations in multiple configurations. 
Fig.~\ref{fig transfer diveser confg} shows the results of LERL-transfer and KCID for circular tracking in 12 distinct configurations, while auxiliary results for target reaching and bow tie tracking in 18 configurations are given in Figs.~\ref{fig appendix 1seg transfer 234segs sp} and \ref{fig appendix 1seg transfer 234segs bowtie}, and Supplementary Video 5. 
\ \\
\textbf{Policy adaptation from one segment to more.} The control policy initially trained in $\text{E}^{1}_\text{S}$ was efficiently transferred using a padding procedure detailed in Supplementary Section 5, and refined online to adapt to new configurations comprising two, three, and four segments. As shown in Fig.~\ref{fig transfer diveser confg}a, LERL-transfer with 2k samples achieves lower tracking errors in all scenarios compared to KCID with 150k samples, reducing the sample size by 75×.
\ \\
\textbf{Policy adaptation under diverse assembly configurations.} 
We demonstrate the rapid policy adaptability to the three-segment elephant-trunk robot with various assembly configurations in Fig.~\ref{fig transfer diveser confg}b, wherein LERL-transfer costing 2k samples outperforms KCID with 20k and 150k samples. Compared to KCID with 150k samples, our method reduces training samples by 75×, while achieving a reduction of 5.51\%-12.64\% in the tracking error. We also evaluate the empirical robustness of LERL in the configuration $\text{E}^{5}_\text{S}$-$\text{E}^{5}_\text{S}$-$\text{E}^{5}_\text{S}$  under different loads (\SI{200}{\, \gram}, \SI{500}{\, \gram}, \SI{1}{\, \kilogram}), as shown in Fig.~\ref{fig transfer diveser confg}e and Supplementary Video 6. The position and angle of the robot tip eventually converge to the target with a steady-state position error of 0.48 mm under a payload of 200 g and with a steady-state position error of 10.07 mm under a payload of \SI{1}{\, \kilogram} (Fig.~\ref{fig different embedding comp}c).
\ \\
\textbf{Policy adaptation to soft-muscle and hybrid-assembled robots.} We further validate LERL’s versatility by transferring the control policy to heterogeneous soft-muscle robots and hybrid-assembled robots (Fig.~\ref{fig transfer diveser confg}c). Although the actuator numbers of soft-muscle segments differ from those of honeycomb segments, the learned policies are directly transferable via an action re-allocation strategy (Supplementary Section 5). The hybrid-assembled robots, composed of honeycomb and soft-muscle segments, exhibit substantial differences in materials, structures, and actuations (Fig.~\ref{fig natural}b and c). Nevertheless, LERL-transfer with only 2k samples enables hybrid-assembled robots to successfully complete tracking tasks, achieving an 8.10\% reduction in tracking error compared to KCID with 150k samples (Fig.~\ref{fig transfer diveser confg}c). These results highlight LERL’s potential for rapid policy adaptation across heterogeneous robots with distinct dynamic characteristics.
\ \\
\textbf{Policy adaptation to worm robots.}
During policy adaptation to the worm segment $\text{E}_\text{M}$, the Koopman embedding pretrained on the honeycomb segment is reused. An offline policy initialization procedure (see Supplementary Section 7) was conducted using 10k samples, as direct policy transfer from the elephant-trunk robot is infeasible due to different actuation modes (positive vs. negative pressure; Fig.~\ref{fig natural}b and d). 
As shown in~Fig.~\ref{fig transfer diveser confg}d, LERL-transfer with 12k samples (10k were used for offline policy initialization) outperforms KCID with 20k and 150k samples. 
In policy adaptation to worm robots with two and three segments, the online learned policy using 2k samples outperforms KCID with 150k samples, reducing the sample size by 75× (Fig.~\ref{fig transfer diveser confg}d). 
Furthermore, in the worm robot with configuration $\text{E}_\text{M}$-$\text{E}_\text{L}$, LERL achieves position tracking with a steady-state position error of 3.22 mm while operating at the extreme condition (Fig.~\ref{fig transfer diveser confg}f) with velocities and accelerations up to $\SI{1.89}{\, \meter \cdot \second^{-1}}$ and $\SI{22.34}{\, \meter \cdot \second^{-2}}$ (Table~\ref{tab related works}), which are 24\% and 92\% higher than the previously recorded\upcite{haggerty2023control} velocity of $\SI{1.52}{\, \meter \cdot \second^{-1}}$ and acceleration of $\SI{11.6}{\, \meter \cdot \second^{-2}}$.
  \begin{figure*}[!t]
     \centering
      \includegraphics[width=0.9\textwidth]{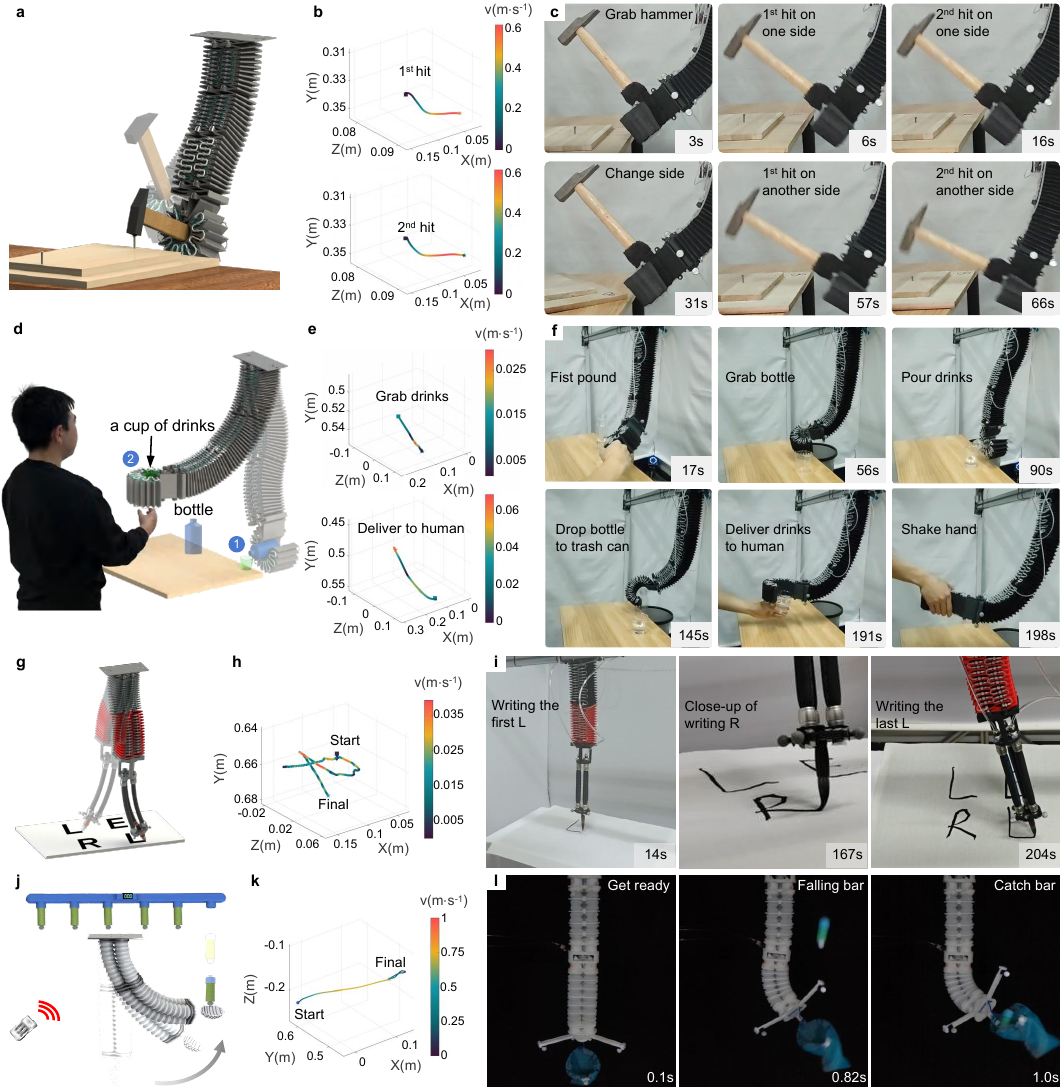}
     \caption{
     \textbf{
     Completing challenging tasks with LERL-empowered soft robots. 
     } 
     \textbf{a}-\textbf{c}, The elephant-trunk soft robot drives nails to bind two wooden boards: scene \textbf{a}, highlighted trajectory \textbf{b}, and snapshots \textbf{c}. 
      \textbf{d}-\textbf{f}, The elephant-trunk robot autonomously serves drinks to a human with safe human-robot interactions: scene \textbf{d}, highlighted trajectory \textbf{e}, and snapshots \textbf{f}. 
    \textbf{g}-\textbf{i}, The hybrid-assembled robot performs calligraphy brush handwriting: scene \textbf{g}, highlighted trajectory \textbf{h}, and snapshots \textbf{i}. 
     \textbf{j}-\textbf{l}, The worm robot mimics human to play a quick hand-eye reaction game: scene \textbf{j}, highlighted trajectory \textbf{k}, and snapshots \textbf{l}. 
     }
     \label{fig complex task}
 \end{figure*} 

\subsection*{Hammering nails}
\vspace{-2mm}
No previous work on soft robots has successfully utilized a hammer to drive nails for binding wooden boards,
like a skilled carpenter. This challenging task demands high payload capacity and precise control at high speeds to generate sufficient force for driving nails. The LERL policy, transferred from the segment $\rm E_{S}^1$ with online updates, has achieved precise control of a 370 g weighted hammer grasped by the elephant-trunk robot (Fig.~\ref{fig complex task}a), striking nails to bind two wooden boards. During the experiment, the robot tip's pose and velocity were measured using a motion-capture system. The measured data was then used as input to the policy network to generate control actions. As shown in Figs.~\ref{fig complex task}b-c and Supplementary Video 1, at the beginning of the task, the elephant-trunk robot was controlled to grip the hammer handed to it by a human. The robot then used the hammer to quickly drive nails into both sides of the wooden board, completing the task in 72 seconds, reaching a maximum speed of $\SI{0.78}{\, \meter \cdot \second^{-1}}$, and achieving an acceleration of $\SI{11.6}{\, \meter \cdot \second^{-2}}$. This task expands the capabilities of soft robots for daily assistive and industrial applications.

\subsection*{Serving drinks with safe interactions}
\vspace{-2mm}
With the rise of embodied intelligence, there has been growing interest in enhancing the intelligent manipulation capabilities of robotic arms,  focusing primarily on rigid body designs\upcite{gupta2021embodied}.
However, soft robots provide safer human-robot interactions compared to rigid robots, as rigid bodies with intelligent algorithms do not ensure absolute safety. Our transferred LERL policy enabled the elephant-trunk robot to perform the drinks-serving task with precise and smooth operation (Fig. \ref{fig complex task}d), emulating the actions of a skilled human bartender.   
As demonstrated in Figs. \ref{fig complex task}e-f and Supplementary Video 2, the soft robot interacts with the human safely through a fist bump and handshake, adaptively adjusts its position and orientation to pour drinks smoothly into the cup, throws the cup into the trash can, and hands the cup to the guest gently and safely. In addition, the robot can maintain stable operation under external human intervention (Supplementary Video 2) through online policy updates. This task demonstrates our approach's capability in adaptive control of soft robots under dynamic loads, highlighting the potential in daily delicate task operations that require safe human-robot interactions. 

\subsection*{Calligraphy brush handwriting}
\vspace{-2mm}
Calligraphy handwriting with soft arms and brushes is a challenging task. It requires the soft arm and brush to precisely adjust their movements to adapt to the dynamic pressure, speed, and angle required for fluid and clear strokes, while ensuring stability in each brushstroke. 
Our transferred LERL policy enables the hybrid-assembled soft robot to achieve smooth handwriting (Fig.~\ref{fig complex task}g), overcoming the challenges of high-precision control posed by dynamic variations across different segment types. It ensures that the brush tip, mounted at the robot's end, accurately and stably hovers at a specific height above the paper while precisely following the character strokes.  This avoids distortion or blurring of characters due to excessive force (caused by low height) or uneven speed, as demonstrated in Fig.~\ref{fig complex task}h-i and Supplementary Video 3. This result highlights the capability of our framework to enable hybrid soft robots with soft brushes to achieve handwriting capabilities previously demonstrated only by rigid robots with rigid pens.

\subsection*{Quick hand-eye reaction game}
\vspace{-2mm}
The control capabilities under high-speed conditions were evaluated on the worm robot for a challenging hand-eye reaction game (Fig.~\ref{fig complex task}j).
This task is commonly used by athletes and children to train hand-eye reaction and coordination speeds.
An autonomous machine equipped with sticks is hanging at a height of 0.48 meters above the robot, and the sticks are free to fall in an unknown and random order. A camera detects the movements of sticks to determine which stick falls. The soft robot is equipped with a net pocket at its tip, which moves quickly to catch the free-fall stick within 100 ms.  The robot with the transferred LERL policy reached speeds of up to $\SI{1}{\, \meter \cdot \second^{-1}}$ and intercepted all free-fall sticks (Figs.~\ref{fig complex task}k-l, Supplementary Video 4). This result demonstrates the potential of our LERL framework for challenging tasks that require high-performing control in high-speed conditions.

\section*{Discussion}
\vspace{-2mm}

This work presents a generalizable control system that achieves unprecedented policy adaptation across 30 distinct soft robot configurations, effectively bridging the longstanding gap between configuration versatility and scalable real-world deployment. By overcoming key limitations in configuration-specific modelling and control, our system addresses one of the primary barriers to the practical application of soft robotics\upcite{ren2024design}.

Conventional model-based approaches demand extensive dynamical remodelling for each new configuration, while model-free RL methods typically require either data-intensive real-world interactions or high-fidelity simulators, which are especially challenging to construct for soft robots due to their nonlinear and configuration-dependent dynamics. In contrast, our method achieves rapid adaptation without configuration-specific modelling or controller redesign, enabled by three critical factors. First, the fundamental dynamic characteristics of soft robots are encoded into a unified, pretrained linear latent space, transforming configuration-dependent nonlinear dynamics into structured, approximately linear variations within this shared space. Second, these variations are efficiently addressed through an online, data-efficient RL algorithm, eliminating the need for configuration-specific remodelling and retraining. This approach achieves superior performance while requiring 75× fewer training samples than prior methods, and operates with real-time efficiency (see Table~\ref{tab computational time}). Third, our control policy is directly learned and updated within the linear embedding space, mitigating the modelling–control gap that often undermines indirect optimal control strategies. As detailed in our theoretical analysis (Supplementary Section 8), this direct formulation substantially improves adaptation speed and control performance.

The presented framework enables fast multi-configuration adaptation across 30 configurations, demonstrating capabilities ranging from high-payload manipulation (\SI{1}{\, \kilogram}) to rapid locomotion ($\SI{1.89}{\, \meter \cdot \second^{-1}}$) and human-like skills such as bartending, carpentry, brush calligraphy, and quick hand-eye reaction. Unlike existing configuration-specific controllers, our method dynamically updates policies online in response to configuration changes, emulating the adaptability observed in natural organisms.  
 We believe this paradigm of online RL within a pretrained Koopman embedding space has the potential to establish future ``foundation” control strategies of heterogeneous robotic systems. This includes unified and adaptable control of both rigid and soft robots (see Supplementary Section 6 for an elementary
validation of rigid robots). 

The demonstrated effectiveness on multiple platforms suggests a broad impact potential in the medical, industrial, and exploratory domains, laying the foundation for the development of agile, embodied intelligent systems\upcite{rus2015design} capable of configuration adaptability. Furthermore, the real-time efficiency of policy training facilitates rapid deployment in time-sensitive applications such as surgical robotics, where rapid adaptation is essential for operational efficiency.

\clearpage




 \section*{Methods}

 \subsection*{\normalsize{Soft robotic platforms}} 
 \vspace{-2mm}
We developed three distinct types of soft robotic platforms for experimental validation, each differing in materials, stifnesses, structures, and actuation modes.
 \ \\
 \textbf{Type I soft robotic platform.}
The Type I soft robotic platform features an elephant-trunk-shaped arm with a modular honeycomb segment architecture. We engineered 22 unique configurations by systematically varying four key parameters: segment length, stiffness, quantity, and assembly sequence (Fig.~\ref{fig all config}). For motion control implementation, the Type I soft robotic platform (Fig.~\ref{fig robot elect system}a) integrates a motion capture module (Luster Swift30A) to measure the pose and velocity of the marked robot tip. The controller is implemented in Python on a CPU (Intel Core i9-13900HX), calculating control actions (50Hz) with measured real-time data related to the robot tip. These control actions are sent via a dSPACE MicroLabBox\upcite{wang2023dexterous} to proportional valves (SMC ITV2030-212CL), which regulate the airflow to inflate airbags, driving the arm's deflection and extension. 
\ \\
\textbf{Type II  soft robotic platform.}
 Each segment of the Type II robot consists of three parallel pneumatic muscles. Type II soft robotic platform includes three soft-muscle robots and two hybrid soft robots assembled using honeycomb and soft-muscle segments in varying orders and quantities (Fig.~\ref{fig all config}). The Type II soft robotic platform shares the same motion capture module and computing unit as the Type I platform (Fig.~\ref{fig robot elect system}a).
\ \\
\textbf{Type III soft robotic platform.}
Type III robots use a worm-like segment structure, including three configurations ranging from one to three segments (Fig.~\ref{fig all config}). The Type III soft robotic platform (Fig.~\ref{fig robot elect system}b) utilizes the same motion capture module but implements the controller in C++ on a CPU (Intel Core i9-13900HX). The calculated actions (50 Hz) are sent to an Arduino Mega 2560 microcontroller, generating analog signals to control vacuum proportional valves (HETIAN ITV2090-212CL5) connected to a vacuum pump. The valve control modulates the air chamber pressure, enabling dynamic arm deflection. 
        
 \subsection*{\normalsize{Overview of the LERL methodology}}
 \vspace{-2mm} 
Our three types of soft robotic platforms,  encompassing 30 distinct configurations (Fig.~\ref{fig all config}), are capable of executing diverse complex tasks. Although the fundamental control objective remains consistent across configurations—precise trajectory tracking through robotic tip manipulation—developing a generalizable controller with multi-configuration adaptability presents substantial challenges. Prior approaches, including model-based controllers\upcite{Bruder9477047,bruder2020data,pan2023auto,wang2022improved,haggerty2023control} and data-driven RL methods\upcite{satheeshbabu2020continuous,nazeer2024rl,naughton2021elastica}, demonstrate limitations in policy adaptation due to labor-intensive remodelling requirements and persistent sim-to-real transfer gaps.  

Here, we present a data efficient LERL framework (Fig.~\ref{fig:control-structure}), which captures intricate dynamic characteristics through Koopman operators and employs online model-free RL to adapt to dynamic variations across multiple configurations.
This framework offers two key features: First, it enables online model-free policy learning in the linear Koopman embedding space, facilitating rapid and data-efficient adaptation to new robot configurations without the computational burden of dynamics remodelling or parameter retraining. Second,  conventional modelling and control approaches often neglect the modelling errors that are intrinsic parts of robot dynamics or treat them merely as disturbances to be rejected, degrading control performance. In contrast, our method implicitly incorporates these modelling errors through direct policy optimization (as shown later in equation~\eqref{Eqn:u-optimal_o}), improving rather than degrading control performance.

Our methodology begins with data-driven training of the Koopman embedding from a single honeycomb segment, which maps the tip's pose and velocity into a linear embedding space (Fig.~\ref{fig:control-structure}a). 
Building upon this, a data efficient RL algorithm is deployed to update control policies in real-time to address configuration-specific dynamic variations, enabling rapid adaptation across diverse configurations via a base policy with online updates (Fig.~\ref{fig:control-structure}b). 

\subsection*{\normalsize{Offline Koopman embedding training}}
\vspace{-2mm}
This section outlines the offline training process for the shared Koopman embedding, which captures the fundamental dynamics of soft robots. First, training data from a single honeycomb robot segment are collected and normalized. Then, we construct the Koopman embedding function using an MLP and introduce the associated loss functions along with control-oriented regularizations to guide the training of the Koopman embedding.
\ \\
\textbf{Training data generation.}
We collected 20k samples of the input-state pair $\{x, u \}$ from the single honeycomb segment $\text{E}^{1}_\text{S}$ under sinusoidal excitation. The control input $u \in \mathbb{R}^{4}$ regulates the voltage applied to four proportional valves. The state of the robot tip is $x= \left[p_x, p_y, p_z, v_x, v_y, v_z, \theta_x, \theta_y, \theta_z,  w_x, w_y, w_z \right]^{\top} \in \mathbb{R}^{12}$ that captures the complete kinematic and dynamic profile of the robot tip, where $p_{\star}$, $v_{\star}$, $\theta_{\star}$, $w_{\star}$ represent the Cartesian coordinates of position, linear velocity, orientation angle, and angular velocity, respectively ($\star\in x,y,z$).
The state and control are normalized to the intervals $\left[-1,1\right]$ and $\left[0,1\right]$ for training.
\ \\
    \textbf{Embedding construction.}
    Koopman operators typically capture the intricate characteristics of nonlinear dynamics through observations in a linear embedding space. Consider a discrete-time dynamics model of soft robots: $x^+=f(x,u)$, where $x$, $x^{+} \in \mathbb{R}^{n}$ are the current and subsequent states sampled in a time interval $t_s$; the control $u\in\mathbb{R}^m$; 
$f$ is a state transition function. 
Let $u_{\infty}=\{u(i)\}_1^{\infty}$ be the collection of all control inputs in the control space, and $\Gamma$ be a left shift operator such that $u_{\infty}(k+1)=\Gamma u_{\infty}(k)$\upcite{proctor2018generalizing}. 
  An infinite-dimensional Koopman operator $\mathcal{K}$, acting on a scalar-valued observable $\psi(x,u_{\infty})$, is defined as\upcite{proctor2018generalizing,korda2018convergence}	$\mathcal{K} \psi(x, u_{\infty})=\psi(x, u_{\infty})\circ (f(x,u_{\infty}(0)), \Gamma u_{\infty})$. 
  The finite-dimensional approximation of $\mathcal{K}$ is typically derived using dynamic mode decomposition with controls\upcite{korda2018linear}, where a predefined collection of Koopman observables $\Psi(x,u)=(\Psi(x),u)\in\mathbb{R}^{n_{\psi}+m}$ is used, $\Psi(x) \in\mathbb{R}^{n_{\psi}}$ is usually selected as basis functions.
  
  Here, we adopt a trainable MLP for an enhanced representation of the embedding function $\Psi(x)$, with alternative approaches including Gaussian kernel functions and time-delay embeddings investigated in our ablation studies (Fig.~\ref{fig different embedding comp}). Letting $s=\Psi(x)$ be the lifted embedding state, one writes the Koopman model\upcite{zhang2021} as
		\begin{equation}\label{Eqn:linear_p-residual}
	{s}^+=A s+B u+ w(s,u),
		\end{equation}
where $A$ and $B$ are the model parameters to be learned\upcite{korda2018convergence}, $w(s,u)$ is the residual modelling uncertainty. 

Although conventional implementations typically incorporate the state $x$ as prior knowledge in constructing $\Psi(x)$, such methods often result in a dense state transition matrix $A$ with limited spectral regularization\upcite{zhang2021}.  Hence, we introduce a trainable transformation matrix $T\in\mathbb{R}^{12\times 12}$ that projects the state $x$ into a structured eigenspace, allowing for direct manipulation of eigenvalues.
 The resulting Koopman embedding is designed as
\begin{equation}\label{eqn:embedding-mlp}
    \Psi(x)=(Tx,\bar{\Psi}(Tx)),
\end{equation}
where $\bar{\Psi}(Tx)$ is implemented using an MLP with a 12-64-128-64-12 network structure. Each hidden layer employs Rectified Linear Unit (ReLU) activation functions. 
\ \\
\textbf{Loss functions.} 
The standard loss functions in Koopman embedding training include linear regression loss ($L_e$), reconstruction loss ($L_r$), and regularization on the matrix $T$ ($L_{T}$), detailed in Supplementary Section 1. To bridge the gap between regression accuracy and control performance, we introduce two control-oriented regularization terms: (a) eigenvalue regularization for open-loop stability guarantees and (b) singular value regularization for enhanced controllability. 
First, the sparse structure of $A$ is explicitly designed, combining $q$ parameterized real eigenvalues and $(n_{\psi}-q)/2$ conjugate eigenvalue pairs. Specifically, the regularization loss for $A$ is formulated as
\begin{equation}
        L_{A}=\sum_{j\in\mathcal{E}_A} \| \lambda(A)_{j} - \eta_1 \|^{2},
\end{equation}
where  $\lambda(A)_j$ represents the $j$-th eigenvalue of $A$, $\eta_1 <1$;  $\mathcal{E}_{A}$ is the collection of indices associated with eigenvalues that satisfy $\lambda(A)_j>\eta_1$.

With regularization term $L_A$, the eigenvalues of $A$ could be restricted within the unit circle, achieving balance between regression accuracy and system stability. However, the controllability matrix $
\mathcal{C} :=\begin{bmatrix}B &AB&A^2B& \cdots &A^{n_{\psi}-1}B\end{bmatrix}$ may exhibit undesirably small singular values corresponding to large eigenvalues of $A$, particularly those that lie within but approach the unit circle. This phenomenon can degrade closed-loop control performance and lead to suboptimal dynamic control behavior.
Hence,  we introduce an additional regularization loss term to shape the singular value distribution of $\mathcal{C}$, formulated as
\begin{equation}
    L_{\mathcal{C}} = \sum_{j\in\mathcal{E}_\mathcal{C}} \| \sigma(\mathcal{C})_{j} - \eta_2  \|^{2},
\end{equation}
where $\sigma(\mathcal{C})_j$ is the $j$-th singular value of $\mathcal{C}$, $\eta_2>0$;
$\mathcal{E}_\mathcal{C}$ is the collection of indices associated with singular values smaller than $\eta_2$.

The overall loss function used for training the Koopman embedding is as follows:
\begin{equation}\label{eq loss sum}
    L=\alpha_{1} L_e+\alpha_{2} L_r+\alpha_{3} L_{ T}+\alpha_{4} L_{ A}+\alpha_{5} L_{\mathcal{C}},
\end{equation}
where $\alpha_{i}$, $i=1,\cdots,5$ are the tunable parameters to balance each loss function.
Optimizing the loss in equation~\eqref{eq loss sum} generates the configuration-agnostic Koopman embedding $\Psi(x)$  along with the configuration-specific model parameters $A$ and $B$ in equation~\eqref{Eqn:linear_p-residual}. Only $\Psi(x)$ is shared among diverse configurations for online policy adaptation. 
 \subsection*{\normalsize{Online model-free policy learning}}
\vspace{-2mm}
This section presents the online, model-free policy learning approach with the shared configuration-agnostic Koopman embedding, enabling real-time and rapid adaptation to diverse new configurations. We first describe the policy design within the embedding space, followed by the learning process using a model-free $\mathcal{Q}$-learning algorithm. 
\ \\
\textbf{Policy design in the embedding space.}  Prior studies used an approximated nominal Koopman model $\hat{s}^+=A s + Bu,$ to formulate LQR\upcite{haggerty2023control} 
        and MPC\upcite{Bruder9477047} 
       approaches for soft robots, which often neglects the effect of modelling uncertainty $w(s,u)$ existing in the real Koopman model~\eqref{Eqn:linear_p-residual}.
However, even a minor modelling uncertainty $w(s,u)$ can compromise the stabilizability of $(A,B)$\upcite{zhang2021}. Also, the Koopman model that best fits model training data
could be suboptimal for controller design\upcite{formentin2021control}. To address this issue, we directly learn the control policy in the linear embedding space through input-state data.  The modelling uncertainty is hidden in the input-state data used during policy learning, which contributes to the policy optimization (proven in Section
S8).

%
The control policy of LERL consists of a dynamic feedback  policy and a feedforward  policy:
		\begin{equation}\label{Eqn:feedback-policy}
		u=u_e+u_r,
		\end{equation}
		where $u_e = Ks_e$ represents the dynamic feedback policy through the gain matrix $K\in\mathbb{R}^{m\times n_{\psi}}$ acting on the linear embedding space spanned by $\Psi(x)$,
        $s_e=s-s_r$ denotes the lifted state error with $s_r=\Psi(x_r)$ being the lifted reference state; 
         the feedforward policy $u_r$ maintains the robot tip steadily at its reference state $x_r$, which is precalculated offline as introduced in Supplementary Section 2.

  The gain matrix $K$ is learned online by solving an optimal control problem as follows:
	\begin{equation}\label{Eqn:optimiz}
	\min_{u_{e}} V(s_{e,0})=\sum_{k=0}^{+\infty}\gamma^k l(s_{e,k},u_{e,k}),
	\end{equation}
	subject to a Koopman error model ${s}_e^+=A{s}_e+Bu_e+ w_e(s_e,u_e)$ with $w_e(s_e,u_e):=w(s,u)-w(s_r,u_r)$,
    where $l(s_{e,k},u_{e,k})=\|s_{e,k}\|^2_{ Q}+\|u_{e,k}\|^2_R$;
 $Q\in\mathbb{R}^{n_{\psi}\times n_{\psi}}$, 
$R \in\mathbb{R}^{m\times m}$ are symmetric positive definite matrices; $0<\gamma\leq 1$ is a discount factor.
 
To perform model-free RL, the so-called $\mathcal{Q}$-function (state-action value function), i.e., $\mathcal{Q}(z_k) :=l(z_k)+\gamma V(As_{e,k}+Bu_{e,k}+w_{e,k})$, is written as  
	\begin{equation}\label{Eqn:q-funation-iteration}
  \mathcal{Q}(z_k)=z_k^{\top}H_oz_k+G(z_k):=z_k^{\top}{H}z_k,\vspace{2mm}
	\end{equation}
	where $z_k:=[s_{e,k}^{\top} , \ u_{e,k}^{\top}]^{\top}$;
    $H_o\in\mathbb{R}^{(m+n_{\psi})\times (m+n_{\psi})}$ and
        $ H=\begin{bmatrix}
	 H_{ss}& H_{su}\\
	 H_{us}& H_{uu}
    	\end{bmatrix}\in\mathbb{R}^{(m+n_{\psi})\times (m+n_{\psi})}$ are unknown matrices;
the term $G(z_k)$ originates from the modelling uncertainty $w(s,u)$ in equation~\eqref{Eqn:linear_p-residual}. Since it is directly tied to the input-state pair, we leverage $G(z_k)$ to improve control performance, rather than treating it as an uncertainty that degrades performance\upcite{zhang2021}.
%
 Setting $\partial \mathcal{Q}(z_k)/\partial u_{e,k}=0$ with~equation~\eqref{Eqn:q-funation-iteration} yields the analytical expression of the optimal feedback policy associated with $H$:
	\begin{equation}\label{Eqn:u-optimal_o}
      u_{e,k}= - H_{uu}^{-1} H_{us} s_{e,k},
	\end{equation}
	which contains the term $G(z_k)$ for control performance improvement. 
%
\ \\
\textbf{Model-free policy learning.} Here, we apply a model-free $\mathcal{Q}$-learning algorithm to approximate the matrix $H$ in equation~\eqref{Eqn:q-funation-iteration} and generate the optimal feedback policy with equation~\eqref{Eqn:u-optimal_o} in the linear embedding space, directly from the online measured data $\{s_k,u_k\}$ and the target information $\{s_{r,k}, u_{r,k}\}$. 
	We first construct a parameterized approximation of the $\mathcal{Q}$-function in equation~\eqref{Eqn:q-funation-iteration} as
	$\hat {\mathcal{Q}}(z,h)=z^{\top}\hat Hz=h^{\top}{\bm z}$, 
	where ${\bm z} :=z\otimes z$, and $h :={\rm vec}(\hat H)\in\mathbb{R}^{q(q+1)/2}$  is the vectorized form of $\hat H$, $q=n_{\psi}+m$. 
With an initialized $\hat H_0$,  our method updates $\hat H$ so that $\hat {\mathcal{Q}}(z,h)$ asymptotically converges to its optimal value. The optimal feedback policy is then naturally derived from the entries of $\hat H$ following equation~\eqref{Eqn:u-optimal_o}.

At the $i$-th iteration, an estimate of the $\mathcal{Q}$-function from equation~\eqref{Eqn:q-funation-iteration} is calculated as
	$d(z_{i,k},\hat H_i)=l(z_{i,k})+\gamma \|z_{i,k+1}\|_{\hat H_i}^2$.
	Then, $h_{i+1}={\rm vec}(\hat H_{i+1})$ is updated by minimizing the temporal difference loss as follows:
	\begin{equation}\label{Eqn:target-reach}
	\min_{h_{i+1}}\sum_{j=k-1}^{k-N}\|h_{i+1}^{\top}{\bm z}_j-d(z_{i,j},\hat H_i)\|^2.
	\end{equation} 
Let $Z_k=[{\bm z}_{[k-N]}\cdots {\bm z}_{[k-1]}]$ be a collected dataset composed of the past $N$-step input-state pairs, where
	${\bm z}_{[i]}=z_{[i]}\otimes z_{[i]}$. Solving equation~\eqref{Eqn:target-reach} obtains:
	\begin{equation}\label{Eqn:solution-least-square}
	h_{i+1}=(Z_kZ_k^{\top})^{-1}Z_kM_k,
	\end{equation}
	where $M_k=[d({\bm z}_{[k-N]},h_i)\cdots d({\bm z}_{[k-1]},h_i)]^{\top}.$ The inverse operation in equation~\eqref{Eqn:solution-least-square} requires that $Z_kZ_k^{\top}$ is of full rank. 
    To meet this condition, the number $N$ is selected greater than $q(q+1)/2$ and the policy with exploration noise $n_e$  
    is adopted during the initial learning stage.
    	If the matrix $Z_kZ_k^{\top}$ is full rank but ill-conditioned, calculating $h_{i+1}$ may introduce substantial errors. To mitigate this issue in online deployment, we propose two alternative strategies: first, employing warm-start techniques during policy learning (Supplementary Section 3); and second, utilizing a gradient descent-based update rule, such as the Adam optimizer.  
        The implementation procedures are deferred to  Algorithm S1. The theoretical guarantee for the online policy learning outlined above is provided in Supplementary Section 8.
\ \\
\textbf{Integral action in policy learning.}
Offline training the feedforward policy $u_r$ in equation~\eqref{Eqn:feedback-policy} relies on precollected quasi-static motion data for each robot configuration. To facilitate rapid policy generation, limited quasi-static data is used, which may introduce inaccuracies and potentially degrade control performance. To mitigate this issue, we incorporate an integral action into policy learning to compensate for these inaccuracies.
To this end, we define
$\dot q=C_ps_e$,
where $C_p=[I_6\ 0]$ and $q\in\mathbb{R}^6$ consist of the robot tip's position and angle.
Discretizing $\dot q$ with a time interval $t_s$ writes $ q^+=q+t_sC_ps_e$.
Defining an extended state $e_a :=[s^{\top}_e,q^{\top}]^{\top}$, one writes the augmented model as
\begin{equation} \label{Eqn:aug}
    e_a^{+}=A_ae_a+B_au_{e}+E_aw_e,
\end{equation}
where $
   A_a=[
        A_1^{\top}\ A_2^{\top}
    ]^{\top}$, $A_1=[
        A \ 0
    ]$, $A_2=[
        t_sC_p\ I_6
    ]$, $ B_a=[B^{\top}\ 0]^{\top}$, and $E_a=[I\ 0]^{\top}$.
Implementing Algorithm S1 with this augmented model generates the dynamic feedback policy $u_{e}$ that includes an integral action on the robot tip's pose error to compensate for feedforward inaccuracies.

\subsection*{\normalsize{Policy adaptation across multiple configurations}}
\vspace{-2mm}

This section describes how the control policy learned on the single honeycomb segment, with online updates, can be directly adapted to elephant-trunk robots, soft-muscle robots, hybrid-assembled robots, and worm robots, featured with distinct configurations and stiffnesses.
\ \\
\textbf{Policy adaptation to elephant-trunk robots.} Transferring the policy to new configurations of the elephant-trunk robot does not require extensive remodelling. Only 500 motion data samples are collected under quasi-static conditions to train the feedforward policy $u_r=\Phi(x_r)$ in equation~\eqref{Eqn:feedback-policy}. Here, $\Phi(\cdot)$ is a neural network with a 12-32-64-32-$m$ architecture and Rectified ReLU activation functions (Supplementary Section 2).
The feedback policy $u_e$ in equation~\eqref{Eqn:u-optimal_o}  is then refined online to adapt to new configurations in real-time.
   \begin{figure}[t!]
        \centering		
                \includegraphics[width= 0.85\textwidth]{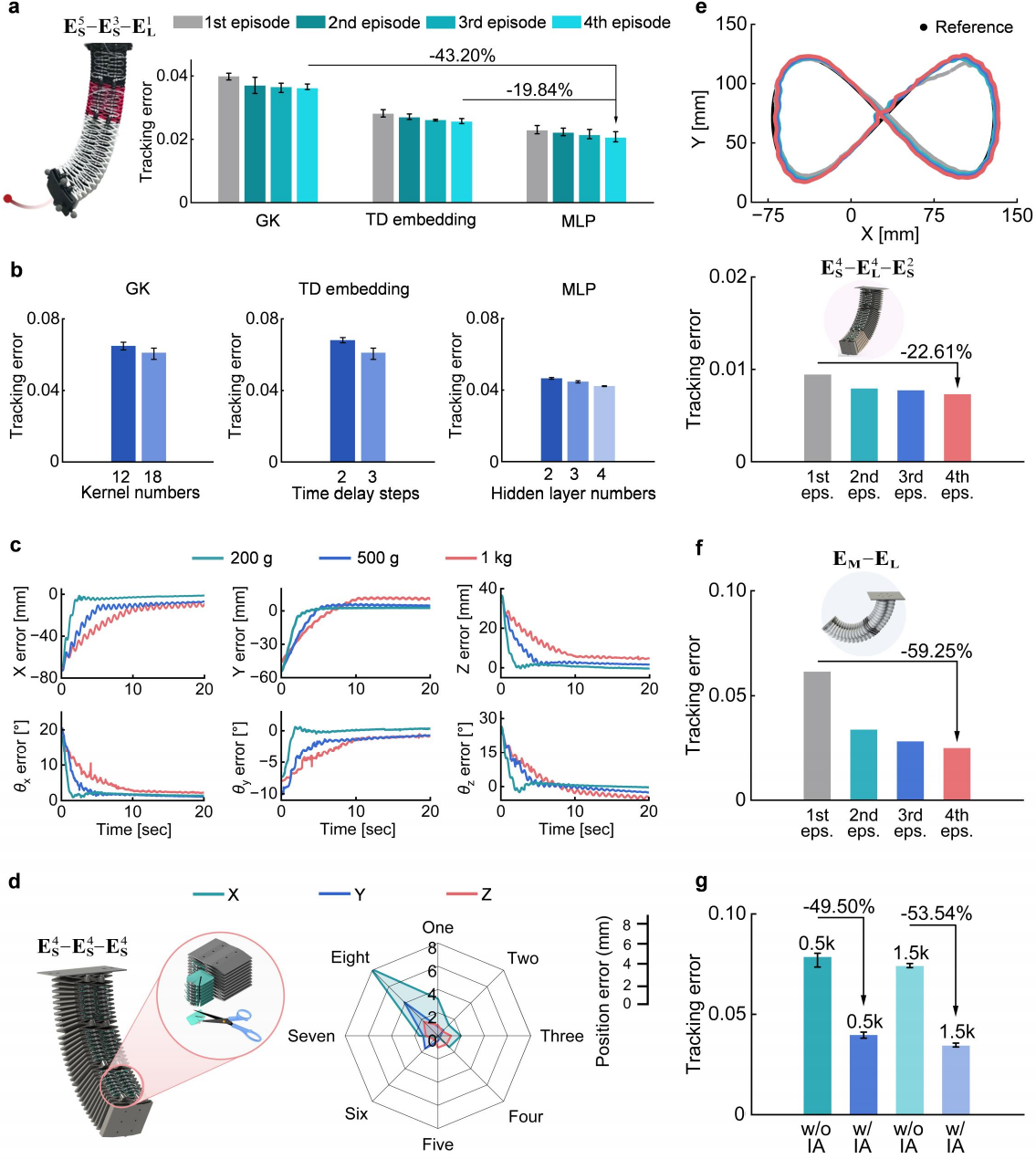}
        \caption{
        \textbf{Validation of learning ability and robustness.} 
        \textbf{a}, Online learning ability validation with different embeddings in target reaching of one elephant-trunk robot.
        \textbf{b}, Control performance under different embedding structures. 
        \textbf{c},  Pose error curves of the robot tip under different load conditions. 
        \textbf{d}, Empirical robustness validation under different numbers of broken airbags, $\text{Position error}:= \frac{1}{M}\sum_{i=1}^M \sqrt{\left\| p_{i,k} - p_{r, i,k} \right\|^2}$, $p_{i,k} = [p_{x, i,k}, p_{y, i,k},p_{z, i,k}]^{\top}$ and $p_{r, i,k} = [p_{r_x, i,k}, p_{r_y, i,k},p_{r_z, i,k}]^{\top}$ denote the real position and its reference at the $k$-th step for the $i$-th experiment, $M = 5$ is the number of experiments.
        \textbf{e}, Four-episode (eps.) learning performance in bow tie tracking of one elephant-trunk robot.
        \textbf{f}, Four-episode learning performance in target reaching of one worm robot.
        \textbf{g},
        Ablation study on integral action (IA) in policy learning. 
        }
        \label{fig different embedding comp}
    \end{figure}
    \ \\
\textbf{Policy adaptation to soft-muscle and hybrid-assembled robots.} Each soft-muscle segment has three control inputs, unlike the four in the honeycomb segment. To enable policy transfer from honeycomb to soft-muscle segments, an action re-allocation strategy is designed to map the four-dimensional control space to three through geometric transformations (Supplementary Section 5). This approach allows for the direct deployment of learned policies to heterogeneous soft-muscle and hybrid-assembled robots without requiring prior modelling.
\ \\
\textbf{Policy adaptation to worm robots.} The control policy of the elephant-trunk robot cannot be directly transferred to the worm robot, as they utilize different actuation modes (positive vs. negative pressure actuation). For policy adaptation to the worm robot, a policy initialization for $u_e$ is employed, utilizing 10k samples to calculate the associated gain matrix $K$ through KCID (Supplementary Section 7).
Subsequently, the learned policy for the configuration $\text{E}_\text{M}$ is transferred and updated online in the linear embedding space to accommodate other configurations of the worm robots. The effectiveness of rapid policy adaptation to heterogeneous worm robots is demonstrated in Fig.~\ref{fig transfer diveser confg}d. However, a key question arises: Is a newly trained Koopman embedding essential for policy transfer to heterogeneous robots? As shown in Fig.~\ref{fig appendix 1seg transfer 234segs sp}, policy transfer with a newly trained Koopman embedding in a worm segment (\text{LERL-new-transfer}) only slightly outperforms that using the pretrained Koopman embedding from the honeycomb segment (\text{LERL-transfer}). These results highlight that our framework enables rapid policy adaptation in heterogeneous robots without notable performance degradation.

\subsection*{\normalsize{Empirical validation}}
\vspace{-2mm}
We conducted comparative studies to verify the effectiveness of our framework, including: (i) online policy learning with different embeddings; (ii) robustness validation under actuator faults and high-payloads; (iii) online multi-episode learning performance; (iv) ablation study on the integral action; and (v) fast-moving traction stick tracking. 
\ \\
\textbf{Learning with different embeddings.}
We evaluated the effectiveness of our approach from two perspectives: the choice of embeddings and the ability to enhance control performance through online learning. As the embedding structure is not unique, we validate the effectiveness of LERL with embedding choices such as MLP, Gaussian kernel (GK), and time-delay (TD) embeddings.  
As shown in  Fig.~\ref{fig different embedding comp}a, online learning reduces tracking error across all adopted embeddings as the learning episodes increase, verifying the effectiveness of online policy learning in performance improvement. The MLP-based embedding achieves the best performance, reducing the tracking error by 43.2\% compared to GK and 19.84\% compared to TD embedding. 
Furthermore, increasing the kernel numbers of GK, time delay steps of TD embedding, and the hidden layers of MLP contribute to the reduction of the tracking error (Fig.~\ref{fig different embedding comp}b).
 \ \\
\textbf{Empirical robustness validation.}
 The empirical robustness of LERL was validated under varying payloads and multiple actuator faults.  Fig.~\ref{fig different embedding comp}c displays pose error curves (averaged over 5 trials) that further validate the robustness under different payloads, complementing the results in Fig.~\ref{fig transfer diveser confg}e. These curves demonstrate LERL's ability to maintain precise tracking in varying load conditions of \SI{200}{\, \gram}, \SI{500}{\, \gram}, and \SI{1}{\, \kilogram}. As shown in  Fig.~\ref{fig different embedding comp}d, LERL sustains $<$2 mm tracking error despite seven broken airbags (Supplementary Video 6). The resilience of LERL to hardware degradation, enabled by online learning, addresses a critical barrier for field applications, where material fatigue or actuator failure is inevitable.
\ \\
\textbf{Online multi-episode learning.}
The multi-episode online learning ability of our method was validated with elephant-trunk and worm robots. 
Our approach enables the elephant-trunk robot in configuration $\text{E}^{4}_\text{S}$-$\text{E}^{4}_\text{L}$-$\text{E}^{2}_\text{S}$  to accurately track circular trajectories (Fig.~\ref{fig different embedding comp}e). 
In addition, the tracking error decreases with each learning episode (reduced by 22.61\% after 4 learning episodes), verifying the online learning capability.
Similarly, for the worm robot configuration $\text{E}_\text{M}$-$\text{E}_\text{L}$, the tracking error in the achievement of the target gradually reduces as the learning episode grows (reduced by 59.25\% after 4 learning episodes), as shown in Fig.~\ref{fig different embedding comp}f.
A visual comparison of performance before and after learning is available in Supplementary Video 7.
\ \\
\textbf{Ablation study on integral action.}
The impact of integral action (IA) on policy learning was evaluated using feedforward policies trained with 0.5 and 1.5k quasi-static motion samples. As shown in  Fig.~\ref{fig different embedding comp}g, incorporating IA improves performance by 49.5\% and 53.54\% in the 0.5k and 1.5k sample scenarios, respectively. However, without IA, increasing feedforward policy training data from 0.5 to 1.5k yields only a marginal reduction in tracking error. 
These results demonstrate that IA effectively compensates for inaccurate feedforward policies and accelerates policy transfer to new configurations
with limited quasi-static motion data.
\ \\
\textbf{Fast-moving traction stick tracking.}
We evaluated LERL's capability in tracking rapidly moving targets using a human-operated traction stick. The experimental results demonstrate LERL's effectiveness in controlling the tip of an elephant-trunk soft robot (configuration $\text{E}^{5}_\text{S}$-$\text{E}^{3}_\text{L}$-$\text{E}^{1}_\text{L}$) to accurately track the fast-moving traction stick.
 Fig. \ref{fig track bar} and Supplementary Video 8 illustrate the system's rapid and precise response to the stick's motion. Additional robustness tests (Supplementary Video 8) confirm LERL's ability to maintain accurate tracking while recovering promptly from external disturbances.

\subsection*{Data availability}
All data needed to evaluate the conclusions in the paper are presented in
the paper and in the Supplementary Information.
Videos and source data are provided with the submission.

\subsection*{Code availability}
The pseudocode for the Koopman embedding pretraining process and the online model-free RL algorithm is available in the GitHub repository \href{https://github.com/xinglongzhangnudt/LERL-for-soft-robots}{https://github.com/xinglongzhangnudt/LERL-for-soft-robots}.

\bibliography{sn-bibliography}
\clearpage
\ \\
\textbf{Competing interests} 
The authors declare that they have no competing interests. 

\ \\
\textbf{Additional information} \\
\textbf{Supplementary information} The supplementary material is attached with the submission. \\
\textbf{Correspondence and requests for materials} should be addressed to Xinglong Zhang, Xin Xu, Dewen Hu.
\clearpage

 \begin{figure*}[!t]
    \centering
    \includegraphics[width=0.9 \textwidth]{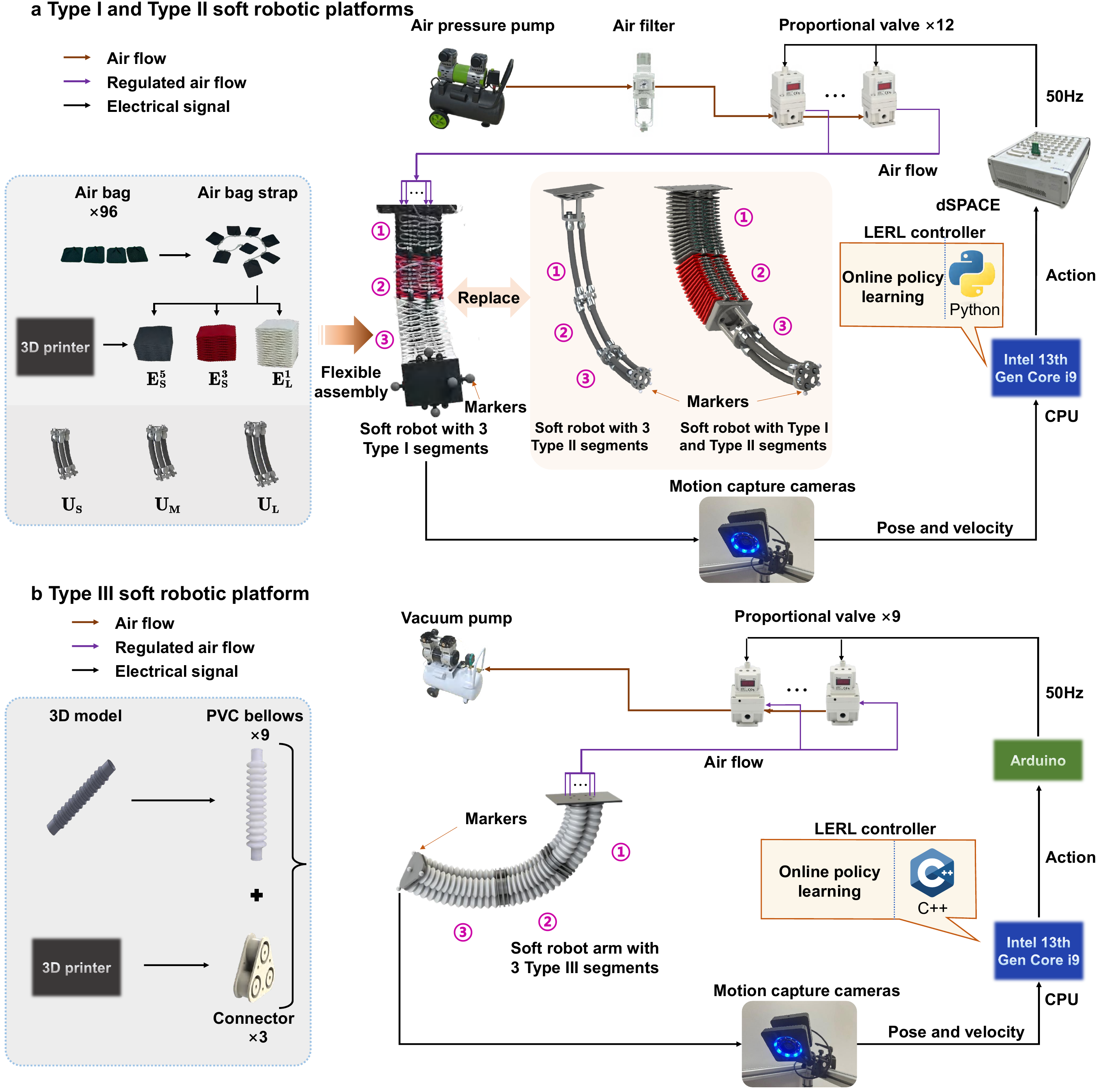}
    \caption{
   \textbf{a}, Type I and Type II soft robotic platforms in specific configurations. Two types of robots share the same motion capture module and the computing unit but with different pneumatic actuation systems across different configurations.
   \textbf{b}, Type III soft robotic platform in a particular setup.
   Three soft robotic platforms, with their multiple configurations (including hybrid-assembled configurations), can perform various tasks.
   }
    \label{fig robot elect system}
\end{figure*}

\clearpage

 \begin{figure*}[!t]
    \centering
    \includegraphics[width=0.9 \textwidth]{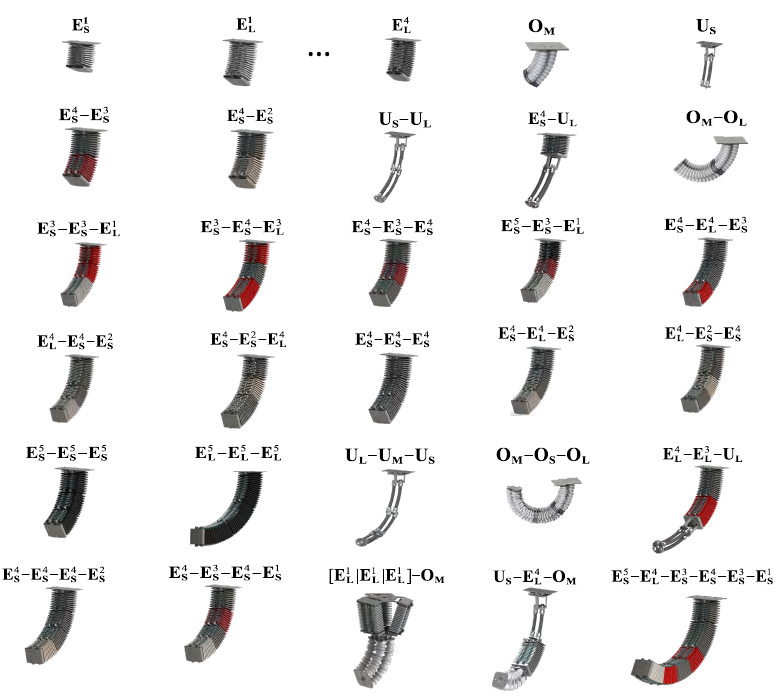}
    \caption{
    \textbf{List of 30 configurations of soft robots for validation.} 
    A library of modular segments includes eight honeycomb segments ($\rm E_S^1$, $\rm E_S^2$, $\rm E_S^3$, $\rm E_S^4$, $\rm E_L^1$, $\rm E_L^2$, $\rm E_L^3$, $\rm E_L^4$), three soft-muscle segments ($\text{E}_\text{S}$, $\text{E}_\text{M}$, $\text{E}_\text{L}$) and three worm segments ($\text{E}_\text{S}$, $\text{E}_\text{M}$, $\text{E}_\text{L}$), varying in length and stiffness. 
These segments are assembled to build 22 configurations of elephant-trunk robots that vary in segment length, stiffness, quantity, and assembly orders; 
three configurations of soft-muscle robots and two configurations of hybrid-assembled robots;
and three configurations of worm robots ranging from one to three segments.
    }
    \label{fig all config}
\end{figure*}

\clearpage

  \begin{figure*}[!t]
    \centering
    \includegraphics[width= 1\textwidth]{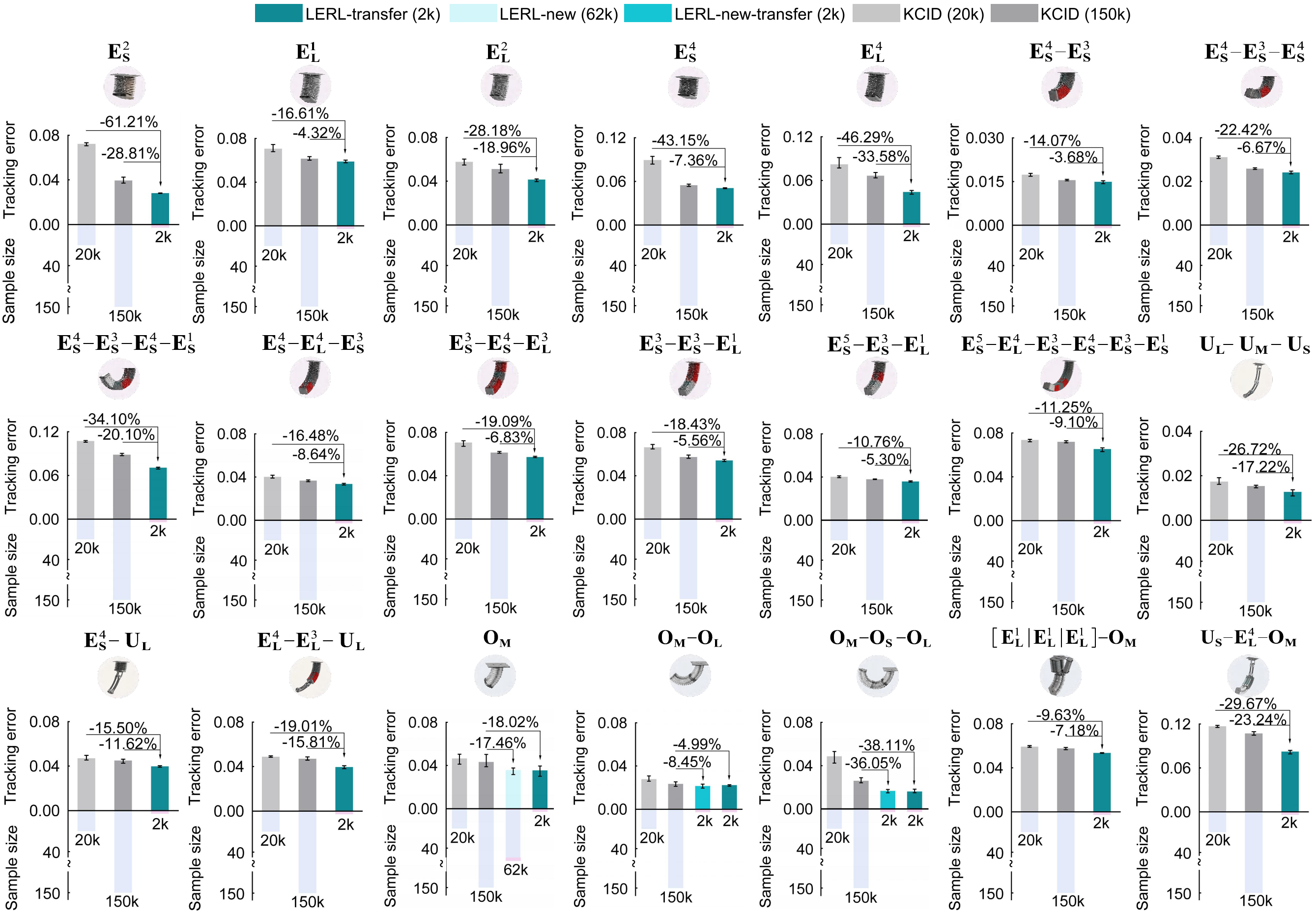}
    \caption{
\textbf{Auxiliary policy adaptation performance across diverse configurations in target reaching tasks.}
Target reaching performance comparisons across 18  configurations from three soft robot prototypes demonstrate that
LERL outperforms KCID in both tracking error and sample size reduction, $\text{Tracking error} := \frac{1}{M \cdot N} \sum_{i=1}^{M} \sum_{k=1}^{N}  \left\| \bar{x}_{i,k} - \bar{x}_{r, i,k} \right\|^2$, $\bar{x}_{i,k}$, $\bar{x}_{r, i,k} \in \mathbb{R}^{12}$ represent the  normalized real state
variable (pose and velocity)  and its reference at the $k$-th step for the $i$-th experiment, $N$ denotes the algorithm horizon, and $M = 5$ is the number of repeated experiments.
The policy adaptation to worm robots' configurations uses policy learning with a newly trained Koopman embedding (LERL-new). The  Koopman embedding for the single worm segment $\text{E}_\text{M}$ is trained with 20k samples, and the policy is updated online from scratch using 2k samples. The policy is then transferred and refined online (LERL-new-transfer) to different configurations with 2k samples under the newly trained Koopman embedding. 
}
    \label{fig appendix 1seg transfer 234segs sp}
\end{figure*}

\clearpage
  \begin{figure*}[!t]
    \centering
    \includegraphics[width= 0.9\textwidth]{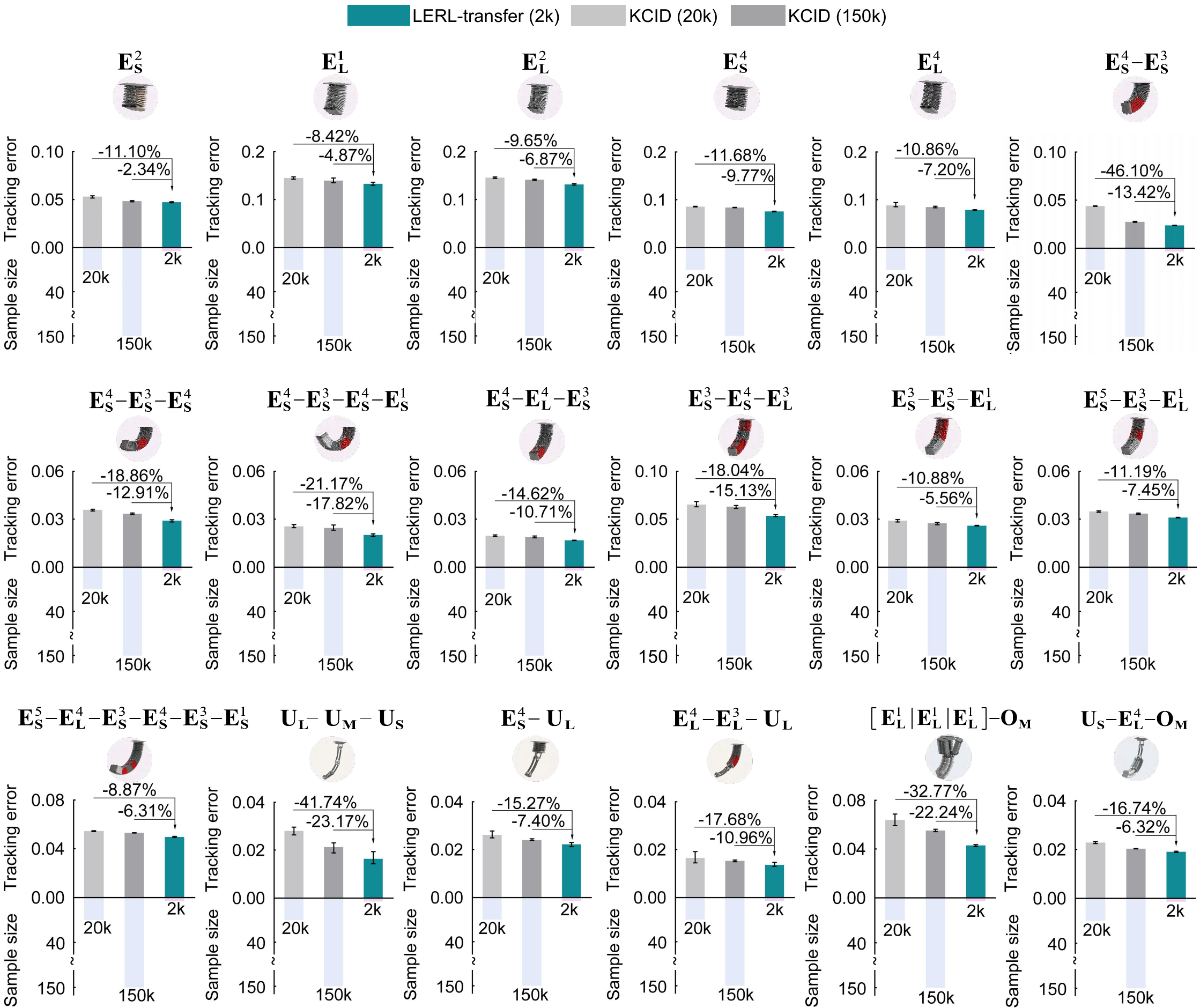}
    \caption{
\textbf{Auxiliary policy adaptation performance across diverse configurations in bow tie tracking tasks.}
Bow tie tracking performance comparisons across 15 configurations from three soft robot prototypes demonstrate that LERL outperforms KCID in both tracking error and sample size reduction.
    }
    \label{fig appendix 1seg transfer 234segs bowtie}
\end{figure*}


\clearpage

 \begin{figure*}[h!t]
    \centering
    \includegraphics[width=0.7 \textwidth]{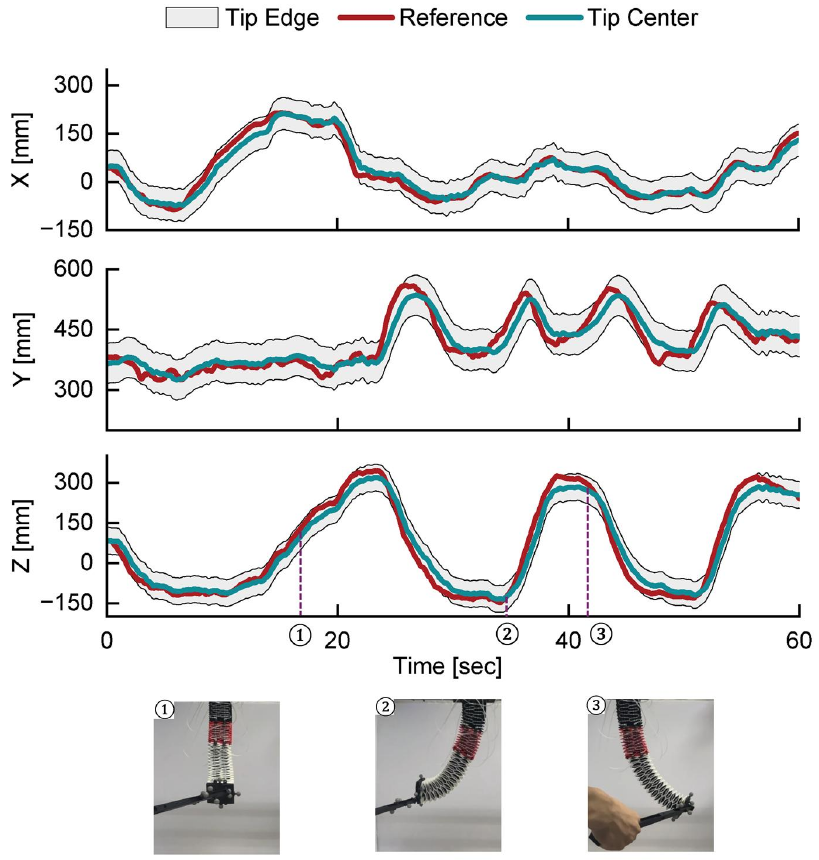}
    \caption{
    \textbf{Fast-moving traction stick tracking performance.}
    Tracking curves in the $x$, $y$, and $z$ directions, along with the pose of the soft robot at selected moments.
    }
    \label{fig track bar}
\end{figure*}
\clearpage

\begin{table*}[h!tb]
        \centering \caption{Computational (comp.) time per step on the Type I elephant-trunk robots.
        The computation is conducted on a laptop equipped with a 13th Gen Intel (R) Core (TM) i9-13900HX processor, featuring a base clock speed of 2.20 GHz.}
        \label{tab computational time}
		\vskip 0.2cm
        \begin{threeparttable}
            \scalebox{0.75}{
            \renewcommand{\arraystretch}{1.2}
            \begin{tabular}{ccccc}
                \toprule
                                        \multirow{1}{*}{\textbf{Robot configuration}}& 
                                        \multirow{1}{*}{\textbf{Input dimension}}
                        &\multirow{1}{*}{\textbf{Lifted space dimension}} 
                        &\multirow{1}{*}{\textbf{Average comp. time (ms)}} 
                        &\multirow{1}{*}{\textbf{Max comp. time (ms)}} 
                     \\
                \midrule
                        \multirow{1}{*}{One segment} 
                        & 4
                        & 24
                        & 1.862
                        & 4.004
                        \\ 
                \midrule
                        \multirow{1}{*}{Two segments} 
                        & 8
                        & 24
                        & 2.429
                        & 5.001
                        \\ 
                \midrule
                        \multirow{1}{*}{Three segments} 
                        & 12
                        & 24
                        & 3.508
                        & 8.001
                        \\ 
                \midrule
                        \multirow{1}{*}{Four segments} 
                        & 16
                        & 24
                        & 3.831
                        & 8.998
                        \\ 
                \bottomrule
	    \end{tabular}
        }
            \end{threeparttable}
    \end{table*}
\clearpage

\makeatletter
\renewcommand{\fnum@figure}{\textbf{Supplementary Fig. \thefigure}}
\renewcommand{\fnum@algorithm}{\textbf{Supplementary Algorithm \thealgorithm}}
\renewcommand{\fnum@table}{\textbf{Supplementary Table \thetable}}
\makeatother

\renewcommand{\theequation}{S\arabic{equation}}
\renewcommand{\thepage}{S\arabic{page}}
\setcounter{figure}{0}
\setcounter{table}{0}
\setcounter{equation}{0}
\setcounter{page}{1} 

\title{
\centering \textbf{Supplementary Information} \\[1ex]
\centering \textbf{Reinforcement learning in linear embedding space unlocks generalizable control} \\
\centering \textbf{across soft robot configurations}\\
}


\maketitle 
\ \newline
\textbf{Supplementary information contains:} 
\ \\
Supplementary Section 1. Koopman embedding training \\
Supplementary Section 2. Feedforward  policy calculation\\
Supplementary Section 3. Policy learning with warm-start\\
Supplementary Section 4. Pseudocode of LERL\\
Supplementary Section 5. Action space re-allocation\\
Supplementary Section 6. Policy adaptation to rigid-body robots\\
Supplementary Section 7. Feedback policy initialization\\
Supplementary Section 8. Theoretical results\\
Supplementary Table 1. Hyperparameters\\
Supplementary Video 1. Hammering nails task\\
Supplementary Video 2. Serving drinks with safe interactions\\
Supplementary Video 3.  Calligraphy brush handwriting\\
Supplementary Video 4. Quick hand-eye reaction game\\
Supplementary Video 5.  Policy adaptation performance validation\\
Supplementary Video 6.  Empirical robustness validation\\
Supplementary Video 7.  Performance improvement from online learning\\
Supplementary Video 8. Fast-moving traction stick tracking\\
\newpage

\clearpage

\subsection*{Supplementary Section 1. Koopman embedding training}
The 20k collected samples are used to train the Koopman embedding on a laptop with an NVIDIA GeForce RTX 4090 GPU. Hyperparameters for offline Koopman embedding training are provided in  Supplementary Table~1.

The regression loss $L_e$ is the mean squared error between the lifted actual and predicted states over $p$ steps, which measures the prediction accuracy in the embedding space, i.e., 
\begin{equation}
    L_e = \frac{1}{p} \sum_{i=1}^{p} \|s_{k+i} - \hat{s}_{k+i} \|^{2},
\end{equation}
where $\hat{s}_{k+i} = A^i s_k + \sum_{j=1}^iA^{j-1}B u_{k+i-j}$ is the $i$-th step ahead state prediction in the linear embedding space from time instant $k$. 

The reconstruction loss $L_r$ is the mean squared error between the actual states and their predicted reconstructed counterparts over $p$ steps, that is,
\begin{equation}\label{eqn:recover}
    L_r = \frac{1}{p} \sum_{i=1}^{p} \|x_{k+i} - \hat{x}_{k+i}\|^{2},
\end{equation}
where the $\hat{x}_{k+i}=T^{-1} s_{k+i}$ is the prediction of the reconstruction state. 


To guarantee the matrix $T$ in~equation~\eqref{eqn:recover} to be invertible, we parameterize $T$ as a trainable tensor and  introduce the following regularization loss: 
\begin{equation}
        L_T= \sum_{i=1}^{12} \| \lambda(T)_{i}-1\|^{2},
\end{equation}
to shape its eigenvalues near the unit circle.

\clearpage

\subsection*{Supplementary Section 2. Feedforward  policy calculation}
A feedforward  policy $u_r$ is incorporated into the control strategy (equation~\eqref{Eqn:feedback-policy} in Methods) to maintain the robot tip in the desired pose when the dynamic feedback policy $u_e = Ks_e$ vanishes. However, this static feedforward  policy is not directly available, as the robot tip's static model is unknown. 
We address this challenge by learning the feedforward policy offline using precollected quasi-static motion data\upcite{haggerty2023control}.
The embedding network, an MLP with an architecture of 12-32-64-32-$m$ and Rectified Linear Unit (ReLU) activation functions, is adopted to generate $u_r$. 
Here, $m$ denotes the input dimension of the current robot configuration, determined by the robot type and the number of segments. 
The Adam optimizer is adopted to train the network parameters offline, using $N_f=500$ quasi-static motion samples of $\{u_r,x_r\}$ to minimize the following mean squared error:
\begin{equation} \label{eq ff policy}
     E_r = \frac{1}{N_f} \sum_{i=1}^{N_f} \|u_{r,i} - \Phi(x_{r,i})\|^{2}.
\end{equation}
The resulting feedforward policy for each robot configuration is represented as $u_r=\Phi(x_r)$, where $\Phi(\cdot)$ is the pretrained MLP tailored for each robot configuration.
\clearpage

\subsection*{Supplementary Section 3. Policy learning with warm-start}
Here we present the implementation of policy learning with warm-start to further enhance learning efficiency. Even in scenarios with nonlinear system dynamics, a stabilizing control policy can still be effectively derived. In this work, a straightforward way is to carry out a model-based control design using offline learned Koopman models, which makes full use of the baseline control policy $u_{e,b}$ and the associated $H_b$ matrix.

Let $\Delta H_i=H_b-\hat H_i$, equation~\eqref{Eqn:target-reach} in Methods is equivalent to solving:
\begin{equation}\label{Eqn:target-reach-1}
\min_{\Delta h_{i+1}}\|\Delta h_{i+1}^{\top}{\bm z}_k-\Delta d(z_k,\Delta H_i)\|^2,
\end{equation} 
where $\Delta d(z_k,\Delta H_i)=h_b^{\top}{\bm z}_k-d(z_k,\hat H_i)$.

Then,  we derive the update for $\Delta h_{i+1}$ as 
\begin{equation}\label{Eqn:solution-least-square-delta}
\Delta h_{i+1}=(ZZ^{\top})^{-1}Z\Delta W,
\end{equation}
where $
\Delta W=[\Delta d({\bm z}_{[1]},\Delta h_i)\cdots \Delta d({\bm z}_{[N]},\Delta h_i)]^{\top}.$
\begin{remark}
Note that the baseline control policy $u_{e,b}$ is not necessarily computed from a model-based controller design. The associated matrix $H_b$ could be estimated by recursively solving equation~\eqref{Eqn:target-reach} with an available control policy $u_{e,b}$.   
\end{remark}

Let a lower bound of the computational residual error of $h_i$ be $\epsilon_{h,i}$.
The following proposition states the advantage of equation~\eqref{Eqn:solution-least-square-delta} over equation~\eqref{Eqn:solution-least-square} in numerical precision.
\begin{proposition}[Numerical precision error]
	If $H_b$ is such that $\Delta H_i\succ 0$  and there exists $\delta>1$ with $\delta \Delta d\leq d$, then the numerical precision error of $h_i$ with the operation in equation~\eqref{Eqn:solution-least-square-delta} is smaller than $\delta^{-1}\epsilon_{h,i}$. 
\end{proposition}
\textbf{Proof}.  Note that the numerical computational error is due to the inverse operation and the weak fulfillment of $ZZ^{\top}$ being full rank. Let the error in operation $(ZZ^{\top})^{-1}$ associated with the $i$-th iteration be $\epsilon_{z,i}$. Then, we have $\epsilon_{h,i+1}=\epsilon_{z,i}ZW$. Similarly, the numerical error of $\Delta h_{i+1}$ is $\epsilon_{\Delta h,i+1}=\epsilon_{z,i}Z\Delta W$. In view of the fact $\delta \Delta d\leq d$ and the definition of $W$ and $\Delta W$, we have $\epsilon_{\Delta h,i+1}\leq\epsilon_{z,i}Z \delta^{-1}W=\delta^{-1}\epsilon_{h,i}$. Note that $\hat H_i=\Delta H_i+H_b$, hence the computation of the vectorization of $\hat H_i$ using the operation in equation~\eqref{Eqn:solution-least-square-delta} results in a numerical precision error smaller than $\delta^{-1}\epsilon_{h,i}$.
\hfill     $\square$

\clearpage

\subsection*{Supplementary Section 4. Pseudocode of LERL}
The main implementation steps of linear embedding RL (LERL) for policy learning across diverse configurations are outlined in Algorithm~\ref{Atm:1}. These steps include offline Koopman embedding pretraining and online policy learning. 
The offline phase establishes baseline capabilities, while the online phase allows adaptive refinement for improved performance. The online update of $h$ can follow three strategies: (a) least-square (LS) rule: update $h$ using equation \eqref{Eqn:solution-least-square}; (b) LS rule with warm-start: update $\Delta h$ using equation~\eqref{Eqn:solution-least-square-delta} and set $h = h_b-\Delta h$; (c) gradient-descent rule: update $h$ using the Adam optimizer.

\begin{algorithm}[ht!]
	\caption{Pseudocode of LERL for multiple configurations}\label{Atm:1}
    	\textbf{Offline design  on single honeycomb segment:}\vspace{-3.5mm}\\
	{\color{black}	\begin{algorithmic}[1]
                \State Collect open-loop input-state data;
			\State  Construct $\Psi(x)$ with equation~\eqref{eqn:embedding-mlp} using an MLP;
			\State  Offline train the Koopman embedding with the loss in equation~\eqref{eq loss sum}; 
            \State  Offline train the feedforward policy  using equation~\eqref{eq ff policy};
	\end{algorithmic}}
    	\dotfill \\
\textbf{Offline design for new configurations:}\vspace{-3.5mm}\\
    \begin{algorithmic}[1]
                \State  Collect 500 input-state samples in the quasi-static condition;
            \State  Offline train the feedforward policy for each new configuration using equation~\eqref{eq ff policy};
			  \vspace{-2mm}
	\end{algorithmic}
	\dotfill \\
    	\textbf{Online policy learning for diverse configurations:}
	\begin{algorithmic}[1]
	\Require $i=0$,  $N_{\rm max}$, $\{x_r\}$. 
	\State  Determine the input dimension $m$ of  the current robot configuration;
        \State Initialize $u_e$ and $\hat H$ with those from the transferred policy;
\While{$i \leq N_{\rm max}$}
\State Measure state $x_i$, generate $s_r$ and $s_{i}$ using the pretrained Koopman embedding;
    \State  Generate the feedforward policy $u_{r}=\Phi(s_r)$ for the current configuration;
\State Calculate $u_{e,i}=-\hat{H}_{uu,i}^{-1}\hat{H}_{us,i}s_{e,i}+n_{e,i}$, in which $n_{e,i}$ is applied in initial several steps;
\State Apply $u_i = u_{e,i} + u_{r}$ to the robot and calculate $d(z_{i},\hat H_i)=l(z_{i})+\gamma \|z_{i}^+\|_{\hat H_i}^2$; 
        \State 	Update $h_{i+1}$; // Choose strategy (c) (Adam optimizer) with limited data, or strategies (a) and (b) (LS rule and LS rule with warm-start) when sufficient data is available. 
\State Construct $\hat H_{i+1}$ with $h_{i+1}$;
        \State Set $i=i+1$;
\EndWhile
	\end{algorithmic}
\end{algorithm}

\clearpage

\subsection*{Supplementary Section 5. Action space re-allocation}
\textbf{Policy transfer to elephant-trunk soft robots.}
We adopt the padding procedure to transfer the feedback policy $u_e = Ks_e = K (\Psi(x)-\Psi(x_d))$, $K\in \mathbb{R}^{4 \times 24}$, trained on the honeycomb segment $\text{E}^{1}_\text{S}$, to elephant-trunk soft robots in multiple configurations. For a configuration with $n$ segments, we transfer the policy gain $K$ by replicating it $n$ times, forming $K_{n,0} =[K;\cdots;K] \in \mathbb{R}^{4n \times 24}$. The policy gain $K_{n,0}$ then initializes the policy learning process for the new configuration and is fine-tuned online.
\ \\
\textbf{Policy transfer to soft-muscle robots.}
Each elephant-trunk segment has four control inputs, while the soft-muscle segment has three. To allow direct policy transfer from the elephant-trunk robot to the soft-muscle robot, we use an action re-allocation strategy. In particular, we first represent the transferred policy from the elephant-trunk robot as a virtual control policy $\bar{u} \in \mathbb{R}^4$ for the soft-muscle robot, which is then mapped to three channels using:
\begin{equation}\label{eq action mappling}
   u = T_a \cdot \bar{u},
\end{equation}
where the action distribution matrix $T_a \in \mathbb{R}^{3 \times 4}$ is derived from equation~\eqref{eq type 3 tran 3} to equation~\eqref{eq type 3 tran 4}.

The above re-allocation strategy allows directly transferring the control policy from the elephant-trunk robot to the soft-muscle and hybrid-assembled robots, without any prior modelling procedure. 
\begin{figure*}[h!t]
    \centering
    \includegraphics[width=0.75 \textwidth]{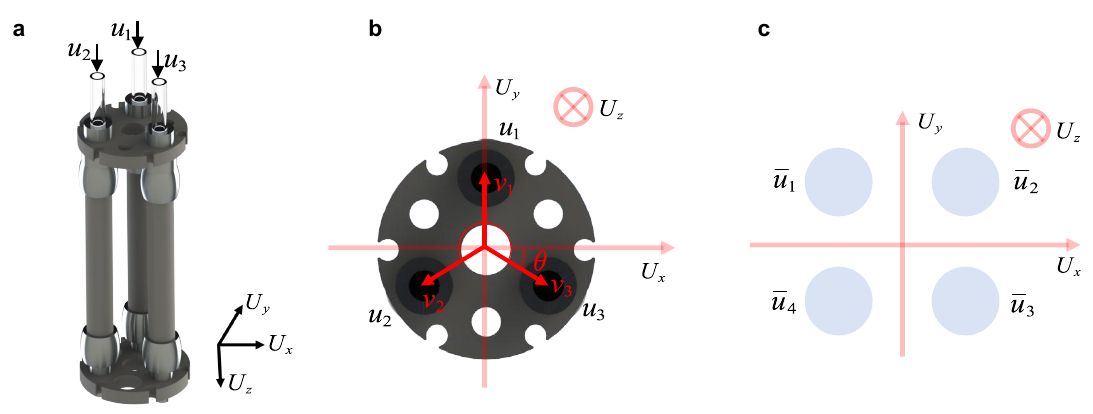}
    \caption{
    \textbf{Mechanism of the action re-allocation strategy.} 
\textbf{a}, Illustration of the equivalent actuation $U$ of a soft-muscle segment.
\textbf{b}, Illustration of the intermediate variable $v$, $\theta = \pi/6$.
\textbf{c}, Illustration of the virtual control $\bar{u}$.
    }
    \label{fig control allocation}
\end{figure*}
\ \\
\textbf{Calculation of matrix $T_a$.} 
The soft-muscle segment has two bending degrees of freedom (DoF) and one elongation DoF, represented by the introduced equivalent actuation $U = [U_x, U_y, U_z]^{\top}$ (Supplementary Fig.~\ref{fig control allocation}a). 
Using the virtual variable $v=[v_1, v_2, v_3]^{\top}$ as a bridge (Supplementary Fig.~\ref{fig control allocation}b), the following matrix $T_{a_{1}}$ is computed to map the real actuation $u = [u_1, u_2, u_3]^{\top}$ to the equivalent actuation $U$, i.e.,
\begin{equation}\label{eq type 3 tran 3}
\begin{array}{ll}
\begin{bmatrix}
    U_x \\
    U_y \\
    U_z \\
\end{bmatrix}
= 
\underbrace{
\begin{bmatrix}
    0 & 3\sqrt{3}/4 & -3\sqrt{3}/4 \\ 
   -3/2 & 3/4 &3/4 \\ 
    1/3 & 1/3 & 1/3 \\ 
\end{bmatrix}
}
\begin{bmatrix}
    u_{1} \\
    u_{2} \\
    u_{3} \\
\end{bmatrix}. \\
\hspace{23mm} T_{a_{1}}
\end{array}
\end{equation}

The virtual actuation $\bar{u}= [\bar{u}_1, \bar{u}_2, \bar{u}_3, \bar{u}_4]^{\top}$ (Supplementary Fig.~\ref{fig control allocation}c) is mapped to the equivalent actuation $U$ through a matrix $T_{a_{2}}$, i.e.,
\begin{equation} \label{eq type 1 tran 1}
\begin{array}{ll}
\begin{bmatrix}
    U_x \\
    U_y \\
    U_z \\
\end{bmatrix}
= \underbrace{\begin{bmatrix}
    1 & 1 & -1 & -1\\ 
    -1 & -1 & 1 & 1\\ 
    1 & 1 & 1 & 1\\ 
\end{bmatrix}}
\begin{bmatrix}
    \bar{u}_{1} \\
    \bar{u}_{2} \\
    \bar{u}_{3} \\
    \bar{u}_{4} \\
\end{bmatrix}.  \\
\hspace{21mm} T_{a_{2}}
\end{array}
\end{equation}
Combing equation \eqref{eq type 3 tran 3} with equation \eqref{eq type 1 tran 1}, one has
\begin{equation} \label{eq type 3 tran 4}
T_a =  T^{-1}_{a_{1}} T_{a_{2}}.
\end{equation}

\clearpage

\subsection*{Supplementary Section 6. Policy adaptation to rigid-body robots}
To further validate the versatility of our LERL framework, we transfer the  policy learned on the soft robotic segment $\text{E}^{1}_\text{S}$ to a three-segment rigid-body robot for a target reaching task. The rigid-body robot is designed in the Mujoco environment (Supplementary Fig.~\ref{fig rigid robot}), with each segment actuated by the virtual force $u_{\text{rig}} = [u_x,u_y]^{\top} \in \mathbb{R}^2$ on the plane vertical to the rigid body. The controlled states, i.e., robot tip's pose and velocity, remain consistent with those in the soft robot scenario. 
We use the action allocation strategy introduced in Supplementary Section 5 to transfer the policy $u \in \mathbb{R}^4$ learned on the honeycomb segment to each segment of the rigid-body robot, i.e.,
\begin{equation}\label{Eqn:reallocation-rigid}
  u_{\text{rig}}= 
  \begin{bmatrix}
     -1/2&-1/2 & 1/2 &  1/2\\ 
     1/2&-1/2 & -1/2 & 1/2\\ 
 \end{bmatrix} u.
\end{equation}
As shown in  Supplementary Fig.~\ref{fig rigid robot}, directly transferring the policy via the control action re-allocation strategy in equation~\eqref{Eqn:reallocation-rigid} leads to instability in the target-reaching task. 
This is because the virtual forces in the rigid-body robot differ in physical interpretation from the air voltage inputs in the honeycomb segment.
However, online policy learning via model-free RL in the shared  Koopman embedding space, pretrained on the honeycomb segment $\text{E}^{1}_\text{S}$,  enables rapid convergence to the target. This demonstrates the effectiveness of our LERL framework in fast adaptation to rigid-body robots. Future work will explore its applicability across various real-world rigid-body robotic systems.
\begin{figure*}[h!t]
    \centering
    \includegraphics[width=0.75 \textwidth]{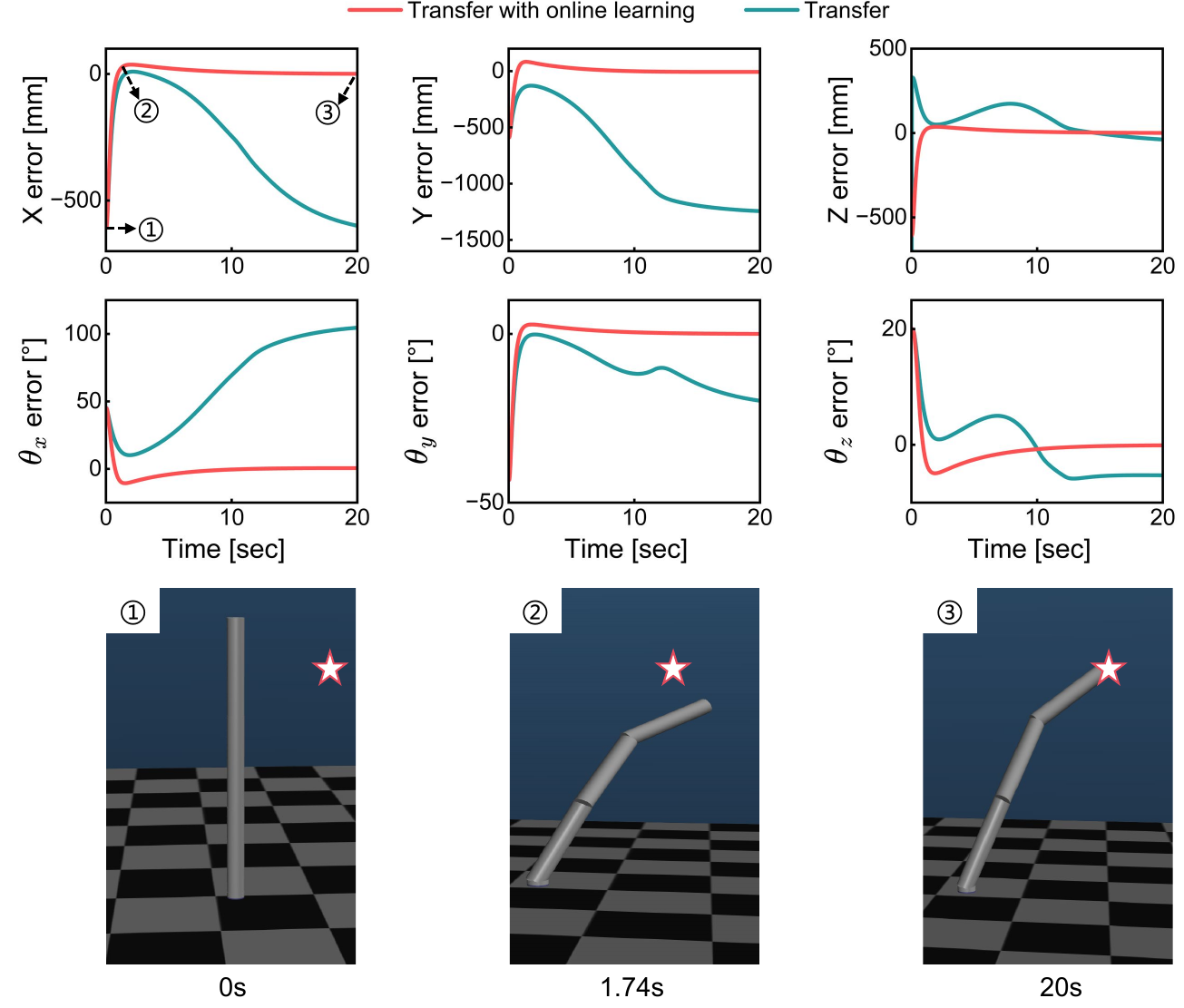}
    \caption{
    \textbf{Policy adaptation to rigid robots.} 
    Pose error curves for policy transfer with or without online learning scenarios, along with the pose of the rigid-body robot at selected moments.
    The highlighted star is the reference target of the robot tip.
    }
    \label{fig rigid robot}
\end{figure*}
\clearpage

\subsection*{Supplementary Section 7. Feedback policy initialization}
The elephant-trunk robot and the worm robot utilize different actuation modes: positive and negative pressure actuation, respectively.
Thereby, a significant gap arises between these heterogeneous robots within the same pretrained Koopman embedding when transferring the policy trained on the honeycomb segment to the worm robots. 

To bridge the aforementioned gap, we effortlessly initialize the feedback policy $u_e$ using just 10k samples $\{x,u\}$ via KCID\upcite{haggerty2023control}. Specifically, given the pretrained Koopman embedding $\Psi(\cdot)$ from the honeycomb segment, the model parameters $A$ and $B$ are first calculated via solving:
\begin{equation} \label{eq ab}
    \min_{A,B} \sum_{i=1}^{N_0}\left\| s_{i+1} - A s_{i}  - B u_i \right\|^2,
\end{equation}
where $s_{i+1} = \Psi(x_{i+1})$ and $s_{i} = \Psi(x_{i})$, $N_0 = 10k$.

Using the calculated  $A$ and $B$, an initialized feedback policy $u_{e} = K_0 s_e$ is generated by solving the standard LQR problem:
\begin{equation}\label{eq lqr}
\begin{aligned}
    \min_{u_e} & \sum_{k=0}^{+\infty} s^{\top}_{e,k} {Q_0} s_{e,k}+ u^{\top}_{e,k} R_0 u_{e,k},  \\  
\text{s.t.} \, \, \, & s_{e,k+1} = A s_{e,k} + B u_{e,k}, 
\end{aligned}
\end{equation}
where  $Q_0\in\mathbb{R}^{n_{\psi}\times n_{\psi}}$ and $R_0^{m\times m}$ are chosen as positive-definite matrices. 
\clearpage

\subsection*{Supplementary Section 8. Theoretical results}
This section theoretically demonstrates that, under mild assumptions, LERL with the least-square update rule~\eqref{Eqn:solution-least-square} outperforms a robust KCID that accounts for modelling errors. 

Assume that the reference pair $\{x_r,u_r\}$ is reachable, i.e., $x_r=f(x_r,u_r)$, where $f(\cdot, \cdot)$ is the unknown dynamics of the soft robot. The dynamic model of the state error is described as $x_e^+=f_e(x_e,u_e):=f(x,u)-f(x_r,u_r)$.  The resulting error dynamics on the embedding space is given as
\begin{equation}\label{Eqn:error_state}
    s^+_e=As_e+Bu_e+w_e(s_e,u_e).
\end{equation}

    According to optimal control theory, a stabilizing feedback policy $u_e=Ks_e$ results in a quadratic value function $V(s_e)=s_e^{\top}Ps_e$ when $w_e(s_e,u_e)=0$ holds, where $P\in\mathbb{R}^{n_{\psi}\times n_{\psi}}$ is a symmetric positive definite matrix. However, since $w_e(s_e,u_e)$ is non-negligible, the feedback policy $u_e$, applied to the model in equation~\eqref
{Eqn:linear_p-residual}, leads to a time-dependent value function $V(s_{e,k})=s_{e,k}^{\top}P_ks_{e,k}$, where $P_k$ varies with time. Hence, the so-called $\mathcal{Q}$-function (state-action value function), i.e., $\mathcal{Q}(z_k) :=l(z_k)+\gamma V(As_{e,k}+Bu_{e,k}+w_{e,k})$, can be rewritten as  
	\begin{equation}\label{Eqn:q-funation-iteration1}
  \mathcal{Q}(z_k)=z_k^{\top}H_oz_k+G(z_k):=z_k^{\top}{H(P_k)}z_k,\vspace{2mm}
	\end{equation}
	where $z_k:=[s_{e,k}^{\top} , \ u_{e,k}^{\top}]^{\top}$;
     $ H_o=\begin{bmatrix}
	 Q+\gamma A^{\top}P_kA& \gamma A^{\top}P_kB\\
	 \gamma B^{\top}P_kA& R+\gamma B^{\top}P_kB
    	\end{bmatrix}$; 
        $ H(P_k)=\begin{bmatrix}
	 H_{ss}(P_k)& H_{su}(P_k)\\
	 H_{us}(P_k)& H_{uu}(P_k)
    	\end{bmatrix}$ is a matrix on $P_k$, expressed in a block form similar to $H_o$;
the term $G(z_k)$ originates from the modelling error. Since it is directly tied to the input-state pair, we leverage $G(z_k)$ to enhance control performance, rather than treating it as an uncertainty that degrades performance\upcite{zhang2021}.

To prove the theoretical property, we introduce the following assumptions about the equivalent Koopman error model in equation~\eqref{Eqn:error_state}.
	\begin{lemma}[Equivalent Koopman error model\upcite{zhang2021}]\label{Eqn:sta-pro}\hfill \\
		(a) Model~\eqref{Eqn:error_state} is stabilizable at $s_e=0$ if and only if $f_e(x_e,u_e)$ is stabilizable. \\
		(b) The model uncertainty $w_e(s_e,u_e)=\Psi(f(\Psi^{-1}(s_e),u_e))-As_e-Bu_e$ satisfies:
		\begin{equation}\label{Eqn:lipsch-1}
		\begin{array}{lll}
		\|w_e(s_e,u_e)\|^2\leq L_s \|s_e\|^2+
		L_u \|u_e\|^2,
		\end{array}
		\end{equation}
		where $L_s$ and $L_u$ are bounded constants. 

\end{lemma}
\begin{assumption}\label{assum-lyap}
	There exists an $\epsilon>0$ such that the following Lyapunov equation has a solution on $\bar P\in\mathbb{R}^{n_{\psi}\times n_{\psi}}$ and  $\bar K\in\mathbb{R}^{m\times n_{\psi}}$, i.e.,
	\begin{equation}\label{Eqn-lyapunov-equa}
	\begin{array}{ll}
	\bar P=&\tilde Q+\bar K^{\top}\tilde R\bar K+\gamma F^{\top}\tilde P F,
	\end{array}
	\end{equation}
	where $\tilde P=\bar P+\epsilon \bar P^2$, $\tilde Q=Q+\gamma (1+\epsilon)/\epsilon L_s$, $\tilde R= R+\gamma (1+\epsilon)/\epsilon L_u$. 
	\begin{remark}
An optimal solution can be calculated by setting $\partial \bar P/\partial K=0$, which leads to: 
\begin{equation}\label{Eqn:policy-value}
\bar K=-\gamma T(\bar P)^{-1}B^{\top}\tilde PA,
\end{equation}
where $T(P)=\tilde R+\gamma B^{\top}\tilde PB$.
Substituting equation~\eqref{Eqn:policy-value} to the right-hand side of equation~\eqref{Eqn-lyapunov-equa}, the result follows:
\begin{equation}\label{Eqn:riccati}
\bar P=\tilde Q-\gamma^2A^{\top}\tilde PBT(P)^{-1}B^{\top}\tilde PA+\gamma A^{\top}\tilde PA,
\end{equation}
which is the variant of the so-called Riccati equation.
	\end{remark}
		\begin{proposition}
		The Hessian matrix of $G(z_k)$, i.e., $\Delta G(z_k)$, is bounded as 
		\begin{equation}\label{Eqn:hessian}
		\Delta G(z_k)\leq Y(P_k),
		\end{equation} 
		where 
		\begin{equation*}
		\begin{array}{l}
		Y(P_k)=
		\begin{bmatrix}\gamma \frac{1+\epsilon}{\epsilon}L_s+\epsilon \gamma A^{\top}P_k^2A&\gamma\epsilon A^{\top}P_k^2B\\
		\gamma \epsilon B^{\top}P_k^2A& \gamma (\frac{1}{\epsilon}+1)L_u+\gamma\epsilon B^{\top}P_k^2B\end{bmatrix}.
		\end{array}
		\end{equation*}
		Moreover, if $s$ spans an invariant Koopman subspace for the dynamic model of soft robots, then $G(z)=0$.
	\end{proposition}
	\textbf{Proof}. 	Note that one can rewrite equation~\eqref{Eqn:q-funation-iteration} as
	\begin{equation}\label{Eqn:value-i}
	\begin{array}{ll}
	\mathcal{Q}(z_k)
	&=s_{e,k}^{\top}Qs_{e,k}+s_{e,k}^{\top}K^{\top}RKs_{e,k}+\gamma(Fs_{e,k}+w_{e,k})^{\top}P_k(Fs_{e,k}+w_{e,k})\\
	&\leq s_{e,k}^{\top}Qs_{e,k}+s_{e,k}^{\top}K^{\top}RKs_{e,k}+\gamma s_{e,k}^{\top}F^{\top}(P_k+\epsilon P_k^2)Fs_{e,k}+
	\gamma\frac{1+\epsilon}{\epsilon}s_{e,k}^{\top} Ls_{e,k},
	\end{array}
	\end{equation}
where the last inequality is due to the following fact:
	Given any constant $\epsilon>0$, it holds that
	\begin{equation*}
	2s_{e,k}^{\top}F^{\top}Pw_{e,k}\leq\epsilon  s_{e,k}^{\top}F^{\top}P^2Fs_{e,k}+\frac{1}{\epsilon}s_{e,k}^{\top}Ls_{e,k},
	\end{equation*}
	where $L=L_s+L_uK^{\top}K$.
	In view of equation~\eqref{Eqn:value-i}, one has:
	\begin{equation*}
	G(z_k)\leq z_k^{\top}Y_k z_k,
	\end{equation*} 
	i.e., the condition in equation~\eqref{Eqn:hessian} holds. Moreover, if $s$ spans an invariant Koopman subspace for the dynamic model for soft robots, the uncertainty in equation~\eqref{Eqn:linear_p-residual} holds $w=0$ and it leads to $\Delta G(z)=0$. \hfill$\square$
\end{assumption}

To state the following Theorems~\ref{theorem 1} and \ref{theo:asymp-stability} in a compact form, let $\bar P^{\ast}$ and $\bar K^{\ast}$ be the optimal solution of equation~\eqref{Eqn-lyapunov-equa} and denote $\bar H_i(\bar P_i)=\Xi_i(\bar P_i)+ Y_i(\bar P_i)$, where $\bar P_i$ is calculated using a value iteration procedure as follows with $\bar P_0=0$.
\begin{enumerate}
	\item Value update:
	\begin{subequations}\label{Eqn:value-iteration}
		\begin{align}\label{Eqn-value-update}
		\bar P_{i+1}=\tilde Q-\gamma^2A^{\top}\tilde P_iBT(\bar P_i)^{-1}B^{\top}\tilde P_iA+\gamma A^{\top}\tilde P_iA.
		\end{align}
		\item Policy update:
		\begin{align}\label{Eqn:policy-i}
		\bar K_{i+1}=-\gamma T(\bar P_i)^{-1}B^{\top}\tilde P_iA.
		\end{align}
	\end{subequations}
\end{enumerate}
	\begin{theorem} \label{theorem 1}
 Under Assumption~\ref{assum-lyap} and $ZZ^{\top}$ is full rank, given $\hat H_0=\bar H_0(P_0)$ and $K_0=\bar K_0$, the iteration of $\hat H_{i}$, using the least-square update rule~\eqref{Eqn:solution-least-square}, satisfies $\hat H_{i}\leq\bar  H_{i}(\bar P_i)$ for any $i\in\mathbb{N}$. Moreover,   $\hat H_{\infty}\leq\bar  H(\bar P^{\ast}):=\bar H^{\ast}$.
	\end{theorem}
	\textbf{Proof.}
	One first rewrites:
	\begin{equation}\label{Eqn:dz-vec}
	\begin{array}{ll}
	d(z_{i,k},\hat H_i)&=l(z_{i,k})+z_{i,k+1}^{\top}\gamma\hat H_iz_{i,k+1}\\
	&\leq{\bm z}_k^{\top}v\left({\Xi}_i(P_{k,i})+ Y_i(P_{k,i})\right),
	\end{array}
	\end{equation}
	where 	
	$ P_i=\begin{bmatrix}
	I&K_i^{\top}
	\end{bmatrix}\hat H_i\begin{bmatrix}
	I&K_i^{\top}
	\end{bmatrix}^{\top}$. Note that the computation of $\bar H_{i}(\bar P_i)$ via the iteration on $\bar P_i$ with equation~\eqref{Eqn-value-update} is equivalent to the iteration on $ H_{i+1}$ under $H_{0}=\bar H_{0}(\bar P_0)$.
	
	Hence, by using the Kronecker products, from equation~\eqref{Eqn:solution-least-square} and equation \eqref{Eqn:dz-vec}, one computes: 
	\begin{equation}\label{Eqn:regression}
	\begin{array}{ll}
	  h_{i+1}&\leq(ZZ^{\top})^{-1}(ZZ^{\top})v\left(\Xi_i(P_{k,i})+ Y_i(P_{k,i})\right)\\
	 &=v\left(\Xi_i(P_{k,i})+ Y_i(P_{k,i})\right)\\ &\leq v\left(\Xi_i(\bar P_{i})+ Y_i(\bar P_{i})\right),
	\end{array}  
	\end{equation}
	where $\bar h={\rm vec}(\bar H)$, the last equality holds since $\hat H_0=\bar H_0$ and $K_0=\bar K_0$.

	In view of the last inequality in equation~\eqref{Eqn:regression}, it holds that: 
	\begin{equation}\label{Eqn: better H}
	\hat H_{i+1}\leq \bar  H_{i+1}(\bar P_{i+1}).
	\end{equation}
	Similar to the proof argument in\upcite{lancaster1995algebraic}, the convergence property of the value iteration procedure as in equation~\eqref{Eqn:value-iteration} can be stated, that is:\\
			 i)$\bar P_1\leq \bar P_2\leq \bar P^{\ast}$; \\
			ii) $\lim_{i\rightarrow +\infty}\bar P_i=\bar P^{\ast}$, $\lim_{i\rightarrow+\infty}\bar K_i=\bar K^{\ast}$. \\
	Hence, $\hat H_{\infty}\leq\bar  H_{\infty}(\bar P_{\infty})=\hat{H}_{\infty}(\bar P_{\infty})+ Y_{\infty}(\bar P_{\infty})=\bar  H^{\ast}$.	
	\hfill $\square$

    \begin{remark}
        Equation~\eqref{Eqn: better H} in Theorem~\ref{theorem 1} shows that the overall cost-to-go with the least-square update rule~\eqref{Eqn:solution-least-square} is smaller than that using the robust value iteration with the procedure~\eqref{Eqn:value-iteration}. The latter results in a robust LQR as iteration $i$ goes to infinity. Hence, LERL with the least-square update rule~\eqref{Eqn:solution-least-square} outperforms the robust KCID with the procedure~\eqref{Eqn:value-iteration}.
    \end{remark}
	
\begin{theorem}\label{theo:asymp-stability}
	Let $\bar s_{e,k}=\gamma^{k/2}s_{e,k}$, then the control input computed by :
	\begin{equation}\label{Eqn:u-converge}
	\bar u_{e,k}=-(\hat H_{uu,\infty})^{-1}\hat H_{us,\infty}\bar s_{e,k},
	\end{equation} stabilizes $\bar s_{e,k}$.
\end{theorem}
\textbf{Poof}.
Consider two Lyapunov functions:
\begin{subequations}
	\begin{align}
	\bar L(s_{e,k})=\|s_{e,k}\|_{\bar P^{\ast}}^2,
	\end{align}
	under $\bar u_{e,k}=\bar K \bar s_{e,k}$, and 
	\begin{align}
	L(s_{e,k})=\|(s_{e,k},\bar u_{e,k})\|_{\hat H_{\infty}}^2,
	\end{align}
	under control equation~\eqref{Eqn:u-converge}.
\end{subequations}
We first show the convergence under the Lyapunov function $\bar L(s_{e,k})$ with $\bar u=\bar K\bar s_{e,k}$. To this end, one writes:
\begin{equation}\label{Eqn:lyp}
L(s_{e,k+1})-L(s_{e,k})=\|s_{e,k+1}\|_P^2-\|s_{e,k}\|_P^2.
\end{equation} 
Consider $\bar s_{e,k}=\gamma^{k/2}s_{e,k}$. In view of equation~\eqref{Eqn:linear_p-residual}, one can write:
\begin{equation}\label{Eqn:bar s}
\begin{array}{ll}
\bar s_{e,k+1}&=\gamma^{1/2}(A\bar s_{e,k}+B\bar u+\bar w_{e,k})\\
&=\gamma^{1/2}(F\bar s_{e,k}+\bar w_{e,k}),
\end{array}
\end{equation}
where the second equality is due to the definition $\bar u_{e,k}=\bar K^{\ast}\bar s_{e,k}$.
Taking equation~\eqref{Eqn:bar s} into equation~\eqref{Eqn:lyp} leads to:
\begin{equation}\label{Eqn:lyp-exten}
\begin{array}{ll}
L(\bar s_{e,k+1})-L(\bar s_{e,k})&=\gamma\|F\bar s_{e,k}+\bar w_{e,k}\|_P^2-\|\bar s_{e,k}\|_P^2\\
&\leq \|\bar s_{e,k}\|_{\gamma (1+\epsilon)/\epsilon L+\gamma F^{\top}(P+\epsilon P^2)F-P}^2\\
&= -\|\bar s_{e,k}\|_{\bar Q+K^{\top}RK}^2\\
&<0.
\end{array}
\end{equation} 
Hence, the state $\bar s_{e,k}$ converges to the origin asymptotically under $\bar u_{e,k}=\bar K^{\ast}\bar s_{e,k}$.

In view of the results that $\hat H_{\infty}\leq \bar H^{\ast}$ and also  note that $\bar L(s_{e,k})=\|z_k\|_{\bar H^{\ast}}^2$, we get:
\begin{equation*}
 L(s_{e,k})\leq  \bar L(s_{e,k}),\ \forall \, k\in\mathbb{N}.
\end{equation*}
 \hfill $\square$
\begin{remark}
	Theorem~\ref{theo:asymp-stability} implies that under the control policy in equation~\eqref{Eqn:u-converge},  the state $s_{e,k}$ converges to the origin asymptotically when $\gamma$ is close to 1.
\end{remark}
\clearpage

\clearpage

\subsection*{Supplementary Table 1. Hyperparameters}

\begin{figure*}[ht!]
    \centering\includegraphics[width=0.7\textwidth]{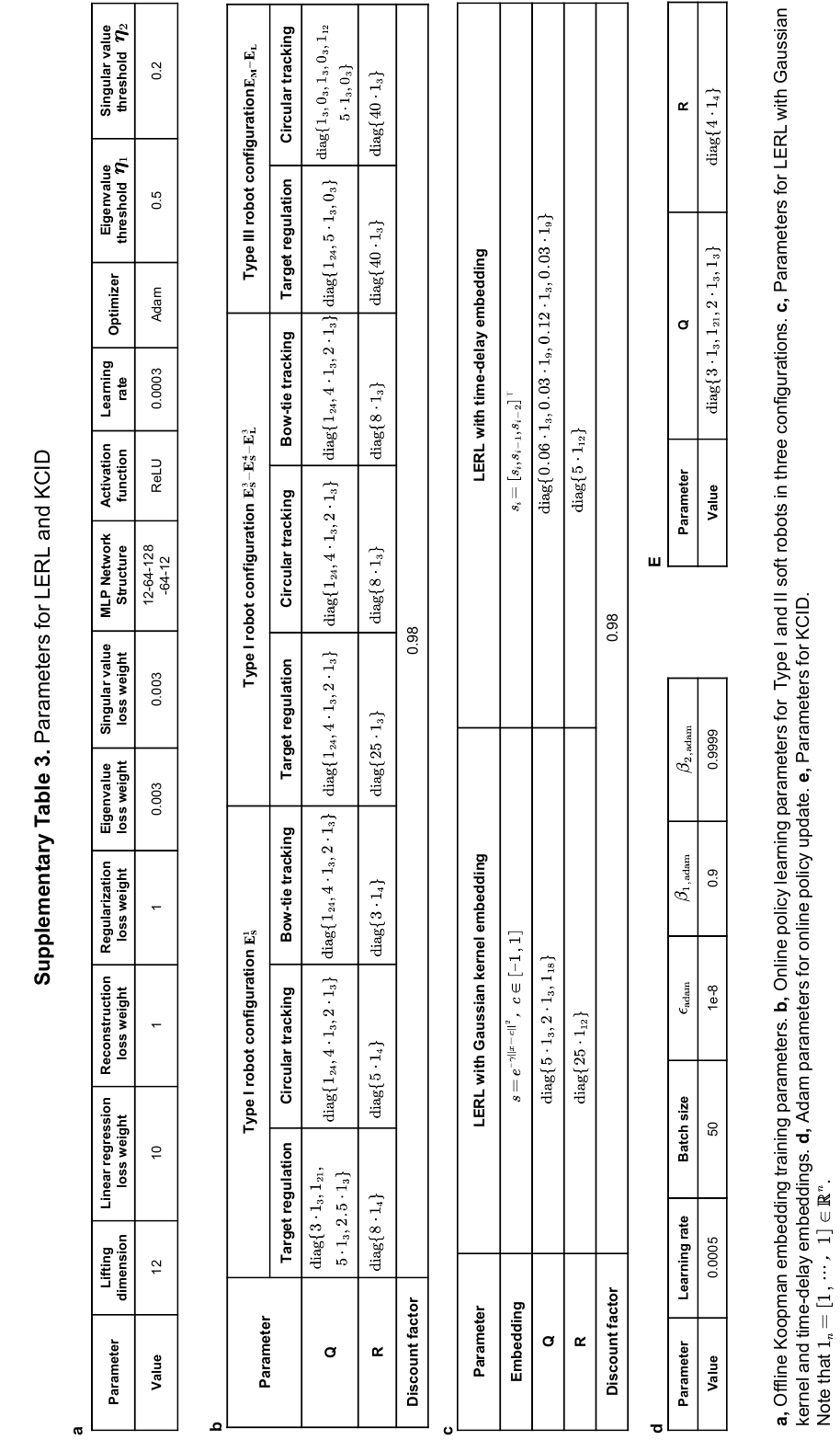}
\label{table parameter}
\end{figure*}

\clearpage

\subsection*{Supplementary Video 1. Hammering nails task}
LERL-enbled elephant-trunk robot strikes nails with a 370 g weighted hammer to bind two wooden boards. At the beginning of the task, the robot was controlled to grip the hammer handed to it by a human. The robot then used the hammer to quickly drive nails into both sides of the wooden board. This video showcases the capabilities of soft robots powered by LERL for daily assistive and industrial applications.
\vspace{-0.2cm}
\subsection*{Supplementary Video 2. Serving drinks with safe interactions}
This video shows that
LERL enables the elephant-trunk robot to complete the drinks-serving task with safe human-robot interactions, emulating a skilled human bartender. The elephant-trunk robot interacts with the human safely through a fist bump and handshake, adaptively adjusts its position and orientation to smoothly pour drinks into the cup, throws the cup into the trash can, and hands the cup to the guest gently and safely.
In addition, the robot remains stable during drinks-serving even under external human interventions. 
This video demonstrates our approach’s capability in precise and smooth control under dynamic
loads of drinks, highlighting the potential in daily delicate task operations that require safe human-robot
interactions.
\vspace{-0.2cm}
\subsection*{Supplementary Video 3.  Calligraphy brush handwriting}
LERL learns from environmental feedback to enable the hybrid soft robot, combined with two honeycomb and one soft-muscle segment, to perform flexible movements, writing characters ``L", ``E", ``R", ``L" with fluid and clear strokes, while ensuring stability in each brushstroke. LERL ensures that the brush tip, mounted at the robot’s end, accurately and stably hovers at a specific height above the paper while precisely following the character
strokes.
\vspace{-0.2cm}
\subsection*{Supplementary Video 4. Quick hand-eye reaction game}
LERL enables the worm robot to play a quick hand-eye reaction game. A camera detects the movements of sticks to determine which stick falls. Then, the worm robot equipped with a net pocket on its tip moves quickly to catch the free-fall stick within 100 ms. This video demonstrates the
potential of our LERL framework for challenging tasks that require high-performing control
in high-speed conditions.
\vspace{-0.2cm}
\subsection*{Supplementary Video 5.  Policy adaptation performance validation}
This video first shows that the online learned control policies drive the honeycomb segment $\text{E}^{1}_\text{S}$ to accomplish target reaching, circle tracking, and bow tie tracking tasks.
Then, its policy is transferred and further updated online to drive three types of soft robots in diverse configurations to complete circle and bow tie tracking tasks. 
This video validates the real-time policy adaptation of LERL to soft robots in multiple configurations.
\vspace{-0.2cm}
\subsection*{Supplementary Video 6.  Empirical robustness validation}
This video presents the empirical robustness validation results of elephant-trunk robots.
LERL enables elephant-trunk robots to achieve precise full-state tracking even under heavy loads of \SI{200}{\, \gram}, \SI{500}{\, \gram}, and \SI{1}{\, \kilogram},  and actuator faults with 22 broken airbags.
LERL’s resilience to hardware degradation addresses a critical barrier
for field applications, where material fatigue or actuator failure is inevitable.
\vspace{-0.2cm}
\subsection*{Supplementary Video 7.  Performance improvement from online learning}
Comparison of control performance before and after learning in elephant-trunk and worm robots. 
Control performance improves substantially throughout the learning process, validating the online learning capability of LERL.
\vspace{-0.2cm}
\subsection*{Supplementary Video 8. Fast-moving traction stick tracking}
The elephant-trunk robot quickly and precisely tracks the fast-moving targets generated from a traction stick manipulated by a human, even under external force interference. This video demonstrates LERL’s effectiveness in tracking rapidly moving targets.

\clearpage

\end{document}